%% file: main.tex
\documentclass{article}
\input{macros}

\title{Towards Deployment of Robust AI Agents \\ for Human-Machine Partnerships}
\author{
    Ahana Ghosh\\
    MPI-SWS\\
    \texttt{gahana@mpi-sws.org}
    \And
    Sebastian Tschiatschek\\
    Microsoft Research\\
    \texttt{setschia@microsoft.com}
    \AND
    Hamed Mahdavi\\
    MPI-SWS\\
    \texttt{hmahdavi@mpi-sws.org}
    \And
    Adish Singla\\
    MPI-SWS\\
    \texttt{adishs@mpi-sws.org}
}

\begin{document}
\maketitle

\newtoggle{longversion}
\settoggle{longversion}{false}

\input{0_abstract}
\input{1_introduction}

\input{3_model}

\input{4.1_algorithmic-ideas}
\input{4.2_performance-analysis}
\input{5.1_algorithm1-Controller}
\input{5.2_algorithm2-AugDQN}
\input{5.3_inference}
\input{6_experiments}
\input{8.0_appendix-aaai-experiments}
\input{2_relatedwork}
\input{7_conclusions}
\subsubsection*{Acknowledgements}
This work was supported by Microsoft Research through its PhD Scholarship Programme.

\bibliography{main}

\iftoggle{longversion}{
\clearpage
\onecolumn
\appendix
{\allowdisplaybreaks
  \input{8.1_appendix-aaai-theorem1}
\input{8.2_appendix-aaai-theorem2}
}
}
{}

\end{document}

%% file: macros.tex
     \usepackage[final, nonatbib]{neurips_2019}

\usepackage[utf8]{inputenc} 
\usepackage[T1]{fontenc}    
\usepackage{hyperref}       
\usepackage{url}            
\usepackage{booktabs}       
\usepackage{amsfonts}       
\usepackage{nicefrac}       
\usepackage{microtype}      
\usepackage{subcaption}

\usepackage{etoolbox}
\usepackage{mathrsfs}
\usepackage{diagbox}
\usepackage{algorithm}
\usepackage{algcompatible}
\usepackage{algpseudocode}

\usepackage{mathtools}
\usepackage{color}
\usepackage{xfrac,amsfonts,amsbsy}
\usepackage[titletoc]{appendix}
\usepackage{xspace}
\usepackage{paralist}
\usepackage{amsthm}
\usepackage{bm}
\usepackage{bbm}
\usepackage{subcaption}

\theoremstyle{definition}
\newtheorem{definition}{Definition}[section]

\newtheorem{theorem}{Theorem}
\newtheorem{lemma}[theorem]{Lemma}

\newtheorem{corollary}[theorem]{Corollary}

\usepackage{enumitem}

\makeatletter

\newcommand{\Rmnum}[1]{\expandafter\@slowromancap\romannumeral #1@}
\makeatother

\DeclareMathOperator*{\argmin}{arg\,min}
\DeclareMathOperator*{\argmax}{arg\,max}

\newcommand{\thetatest}{\ensuremath{\theta^{\text{test}}}}
\newcommand{\agentone}{agent $\mathcal A^x$}
\newcommand{\agenttwo}{agent $\mathcal A^y$}

\newcommand{\woagentone}{$\mathcal A^x$}
\newcommand{\woagenttwo}{$\mathcal A^y$}

\newcommand{\pool}{\ensuremath{\textnormal{\textsc{Pool}}}}

\newcommand{\mdpM}{\ensuremath{\mathcal{M}}}
\newcommand{\rmax}{\ensuremath{r_{\textnormal{max}}}}

\newcommand{\E}{\mathbb E}

\newcommand{\bss}[1]{\left[{#1}\right]}
\newcommand{\bracket}[1]{\left[#1\right]}

\newcommand{\algcontroller}{\textsc{{AdaptPool}}\xspace}
\newcommand{\algdqn}{\textsc{{AdaptDQN}}\xspace}

\newcommand{\norm}[1]{\left\lVert#1\right\rVert}


%% file: 0_abstract.tex
\begin{abstract}
We study the problem of designing AI agents that can robustly cooperate with people in human-machine partnerships.  Our work is inspired by real-life scenarios in which an AI agent, e.g., a virtual assistant, has to cooperate with new users after its deployment. We model this problem via a parametric MDP framework where the parameters correspond to a user's type and characterize her behavior. In the test phase, the AI agent has to interact with a user of unknown type. Our approach to designing a robust AI agent relies on observing the user's actions to make inferences about the user's type and adapting its policy to facilitate efficient cooperation. We show that without being adaptive, an AI agent can end up performing arbitrarily bad in the test phase. We develop two algorithms for computing policies that automatically adapt to the user in the test phase. We demonstrate the effectiveness of our approach in solving a two-agent collaborative task.
\end{abstract}

%% file: 1_introduction.tex
\section{Introduction}

An increasing number of AI systems are deployed in human-facing applications like autonomous driving, medicine, and education~\cite{yu2017towards}.
In these applications, the human-user and the AI system (agent) form a partnership, necessitating mutual awareness for achieving optimal results~\cite{hadfield-menell16cooperative,wilson2018collaborative,Amershi2019Guidelines}. 
For instance, to provide high utility to a human-user, it is important that an AI agent can account for a user's preferences defining her behavior and act accordingly, thereby being adaptive to the user's type~\cite{DBLP:conf/hri/NikolaidisRGS15,DBLP:journals/corr/NikolaidisFHSS17,Amershi2019Guidelines,tschiatschek2019learner,haug2018teaching}.
As a concrete example, an AI agent for autonomous driving applications should account for a user's preference to take scenic routes instead of the fastest route and account for the user's need for more AI support when driving manually in confusing situations.




AI agents that do not account for the user's preferences and behavior typically degrade the utility for their human users.
However, this is challenging because the AI agent needs to (a) infer information about the interacting user and (b) be able to interact efficiently with a large number of different human users, each possibly showing different behaviors.
In particular, during development of an AI agent, it is often only possible to interact with a limited number of human users and the AI agent needs to generalize to new users after deployment (or acquire information needed therefore quickly).
This resembles multi-agent reinforcement learning settings in which an AI agent faces unknown agents at test time~\cite{grover18learning} and the cold-start problem in recommender systems~\cite{Bobadilla2012ColdStart}.


In this paper, we study the problem of designing AI agents that can robustly cooperate with new unknown users for human-machine partnerships in reinforcement learning (RL) settings after deployment.
In these problems, the AI agent often only has access to the reward information during its development while no (explicit) reward information is available once the agent is deployed.
As shown in this paper, an AI agent can only achieve high utility in this setting if it is adaptive to its user while a non-adaptive AI agent can perform arbitrarily bad.
We propose two adaptive policies for our considered setting, one of which comes with strong theoretical robustness guarantees at test time, while the other is inspired by recent deep-learning approaches for RL and is easier to scale to larger problems.
Both policies build upon inferring the human user's properties and leverage these inferences to act robustly.

Our approach is related to ideas of multi-task, meta-learning, and generalization in reinforcement learning.
However, most of these approaches require access to reward information at test time and rarely offer theoretical guarantees for robustness (see discussion on related work in Section~\ref{sec.relatedwork}).
Below, we highlight our main contributions:\\[-1.3em]
\begin{itemize} 
    \item We provide an algorithmic framework for designing robust policies for interacting with agents of unknown behavior. 
    Furthermore, we prove robustness guarantees for approaches building on our framework.\\[-1.5em]
    \item We propose two policies according to our framework: \algcontroller{} which pre-computes a set of best-response policies and executes them adaptively based on inferences of the type of human-user; and \algdqn{} which implements adaptive policies by a neural network in combination with an inference module.\\[-1.5em]
    \item We empirically demonstrate the excellent performance of our proposed policies when facing an unknown user.
\end{itemize}







%% file: 3_model.tex

\section{The Problem Setup}\label{sec.setup}

We formalize the problem through a reinforcement learning (RL) framework. The agents are hereafter referred to as \agentone~and \agenttwo: here, \agenttwo~represents the AI agent whereas \agentone~could be a person, i.e., human user. Our goal is to develop a learning algorithm for \agenttwo~that leads to high utility even in cases when the behavior of \agentone~and its committed policy is unknown. 
\subsection{The model}\label{sec.setup.model}
We model the preferences and induced behavior of \agentone~via a parametric space $\Theta$.
From \agenttwo's perspective, each $\theta \in \Theta$ leads to a parameterized MDP $\mathcal{M}(\theta) := (S, A, T_{\theta}, R_{\theta}, \gamma, \mathcal{D}_0)$ consisting of the following:
\begin{itemize} 
\item a set of states $S$, with $s \in S$ denoting a generic state.
\item a set of actions $A$, with $a \in A$ denoting a generic action of \agenttwo. 
\item a transition kernel parameterized by $\theta$ as $T_{\theta}(s'~|~s, a)$, which  is a tensor with indices defined by the current state $s$, the \agenttwo's action $a$, and the next state $s'$.  In particular, $T_{\theta}(s'~|~s, a)=\E_{a^x}[T^{x,y}(s'~|~s, a, a^x)]$, where $a^x \sim \pi^x_{\theta}(\cdot~|~s)$ is sampled from \agentone{}'s policy in state $s$. That is, $T_{\theta}(s'~|~s, a)$ corresponds to the transition dynamics derived from a two-agent MDP with transition dynamics $T^{x,y}$ and \agentone{}'s policy $\pi^x_{\theta}$.
\item a reward function parameterized by $\theta$ as $R_{\theta}\colon  S \times A \rightarrow [0, \rmax]$ for $\rmax > 0$.  This captures the preferences of \agentone{} that \agenttwo{} should account for.
\item a discount factor $\gamma \in [0, 1)$ weighing short-term rewards against long-term rewards.
\item an initial state distribution $\mathcal{D}_0$. 
\end{itemize}
Our goal is to develop a learning algorithm that achieves high utility even in cases when  $\theta$~is unknown. In line with the motivating applications discussed above, we consider the following two phases:
\begin{itemize}
    \item {\bfseries Training (development) phase.} During development, our learning algorithm can iteratively interact with a limited number of different MDPs $\mathcal{M}(\theta)$ for $\theta \in \Theta^{\textrm{train}} \subseteq \Theta$: here, \agenttwo~can observe rewards as well as \agentone's actions needed for learning purposes.
    \item {\bfseries Test (deployment) phase.} After deployment, our learning algorithm interacts with a parameterized MDP as described above for unknown $\thetatest \in \Theta$: here, \agenttwo~only observes \agentone's actions but not rewards. 
\end{itemize}

\subsection{Utility of \agenttwo}
For a fixed policy $\pi$ of  \agenttwo, we define its total expected reward in the MDP $\mathcal{M}_\theta$ as follows:
\begin{align}
J_\theta(\pi) = \E \bracket{ \sum_{\tau = 1}^\infty \gamma^{\tau-1} R_{\theta}(s_{\tau}, a_{\tau}) ~|~ \mathcal{D}_0, T_{\theta}, \pi},
\end{align}
where the expectation is over the stochasticity of policy $\pi$ and the transition dynamics $T_\theta$.
Here $s_{\tau}$ is the state at time $\tau$.
For $\tau = 1$, this comes from 
the distribution $\mathcal{D}_0$.

\paragraph{For known $\theta$.}
%
When the underlying parameter $\theta$ is known, the task of finding the best response policy of \agenttwo{} reduces to the following:\\
\begin{align}
\pi^*_{\theta} = \argmax_{\pi \in \Pi}  J_\theta(\pi)
\end{align}
where $\Pi = \{ \pi \mid \pi \colon S \times A \rightarrow \bss{0,1} \}$ defines the set of stationary Markov policies.

\paragraph{For unknown $\theta$.}
However, when the underlying parameter $\theta \in \Theta$ is unknown, we define the best response (in a minmax sense) policy $\pi \in \Pi$ of \agenttwo{} as:
\begin{align}
\pi^*_{\Theta} = \argmin_{\pi \in \Pi} \max_{\theta \in \Theta} \Big(J_\theta(\pi^*_\theta) - J_\theta(\pi)\Big)
\label{maxmin_formulation_static}
\end{align}


Clearly, $J_\theta(\pi^*_\theta) - J_\theta(\pi^*_{\Theta}) \geq 0 \ \forall \theta \in \Theta$. In general, this gap can be arbitrarily large, as formally stated in the following theorem.


\begin{theorem}\label{thm_badperformance}
There exists a problem instance where the performance of \agenttwo~can be arbitrarily worse when \agentone's type $\thetatest$ is unknown. In other words, the gap $\max_{\theta \in \Theta} \Big(J_\theta(\pi^*_\theta) - J_\theta(\pi^*_\Theta)\Big)$ is arbitrarily high.
\end{theorem}
The proof is presented in the supplementary material. Theorem~\ref{thm_badperformance} shows that the performance of \agenttwo~can be arbitrarily bad when it doesn't know $\thetatest$ and is restricted to execute a fixed stationary Markov policy. In the next section, we present an algorithmic framework for designing robust policies for \agenttwo~for unknown $\thetatest$.  



%% file: 4.1_algorithmic-ideas.tex
\section{Designing Robust Policies}
In this section, we introduce our algorithmic framework for designing robust policies for the AI \agenttwo{}.
%

\begin{algorithm}[t!]
  \caption{Algorithmic framework for robust policies}\label{alg:algorithm.model}
  \begin{algorithmic}[1]
    \Statex {\bfseries Training phase}  
    \State \emph{Input:} parameter space $\Theta^{\textnormal{train}}$
    \State adaptive policy $\psi \leftarrow \textsc{Training}(\Theta^{\textnormal{train}})$
  \end{algorithmic}
    \begin{algorithmic}[1]
    \Statex {\bfseries Test phase} 
    \State \emph{Input:} adaptive policy $\psi$ 
    \State $O_0 \leftarrow ()$
    \For{$t=1, 2, \ldots$}
      \State Observe current state $s_t$ 
      \State Estimate \woagentone's type as $\theta_t \leftarrow \textsc{Inference}(O_{t-1})$ 
      \State Take action $a_t \leftarrow \psi(s_t, \theta_t)$
      \State Observe \woagentone's action $a^x_t$; $O_{t} \leftarrow O_{t-1} \oplus (s_t, a^x_t)$
    \EndFor
  \end{algorithmic}
\end{algorithm}

\subsection{Algorithmic framework}
Our approach relies on observing the behavior (i.e., actions taken) to make inferences about the \agentone's type $\theta$ and adapting \agenttwo{}'s policy accordingly to facilitate efficient cooperation. This is inspired by how people make decisions in uncertain situations (e.g., ability to safely drive a car even if the other driver on the road is driving aggressively). The key intuition is that at test time, the \agenttwo~can observe \agentone's~actions which are taken as $a^x \sim \pi^{x}_{\theta}(\cdot~|~s)$ when in state $s$ to infer $\theta$, and in turn use this additional information to make an improved decision on which actions to take.
More formally, we define the observation history available at the beginning of timestep $t$ as $O_{t-1} = (s_\tau, a^x_\tau)_{\tau = 1, \ldots, t - 1}$ and use it to infer the type of \agentone{} and act appropriately.


In particular, we will make use of an \textsc{Inference} procedure (details provided in Section~\ref{sec:inference}). Given $O_{t-1}$, this procedure returns an estimate of the type of \agentone{} at time $t$ given by $\theta_t \in \Theta$. Then, we consider stochastic policies of the form $\psi \colon S \times A \times \Theta \rightarrow \bss{0,1}$. The space of these policies is given by $\Psi = \{ \psi \mid \psi \colon S \times A \times \Theta \rightarrow \bss{0,1} \}$. For a fixed policy $\psi$ of \agenttwo{} and fixed, unknown $\theta$, we define its total expected reward in the MDP $\mathcal{M}(\theta)$ as follows:
%
\begin{align}
J_\theta(\psi) = \E \bracket{ \sum_{\tau = 1}^\infty \gamma^{\tau-1} R_{\theta}(s_{\tau}, a_{\tau}) ~|~ \mathcal{D}_0, T_{\theta}, \psi}.
\end{align}
Note that at any time $t$, we have $a_{t} \sim \psi(s_t, \theta_t)$ and $O_{t-1} = (s_\tau, a^x_\tau)_{\tau = 1, \ldots, t - 1}$ is generated according to $a^x_{\tau} \sim \pi^x_{\theta}(s_\tau)$.



We seek to find the policy for \agenttwo{} given by the following optimization problem:
\begin{align}
\min_{\psi \in \Psi} \max_{\theta \in \Theta} \Big(J_\theta(\pi^*_\theta) - J_\theta(\psi)\Big)
\label{maxmin_formulation}
\end{align}

In the next two sections, we will design algorithms to optimize the objective in Equation~\eqref{maxmin_formulation} following the framework outlined in Algorithm~\ref{alg:algorithm.model}. In particular, we will discuss two possible architectures for policy $\psi$ and corresponding \textsc{Training} procedures in Section~\ref{sec:training}. Then, in Section~\ref{sec:inference}, we describe ways to implement the \textsc{Inference} procedure for inferring \agentone's type using observed actions. Below, we provide theoretical insights into the robustness of the proposed algorithmic framework.

%% file: 4.2_performance-analysis.tex
 \subsection{Performance analysis}
 \label{sec:performance-analysis}
We begin by specifying three technical questions that are important to gain theoretical insights into the robustness of the proposed framework, see below:
\begin{enumerate}[leftmargin=0.7cm,label=Q.\arabic*,start=1]
    \item Independent of the specific procedures used for \textsc{Training} and \textsc{Inference}, the first question to tackle is the following: When  \agentone{}'s~true type is $\theta^{\textnormal{test}}$ and \agenttwo~uses a best response policy for $\pi^*_{\hat{\theta}}$ such that $||\theta^{\textnormal{test}} - \hat{\theta}|| \leq \epsilon$,  what are the performance guarantees on the total utility achieved by \agenttwo? (see Theorem~\ref{theorem:approximateMDP}).
    \item Regarding \textsc{Training} procedure: When  \agentone{}'s~type is $\theta^{\textnormal{test}}$ and the inference procedure outputs $\hat{\theta}$ such that $||\theta^{\textnormal{test}} - \hat{\theta}|| \leq \epsilon$, what is the performance of policy $\psi$? (see Section~\ref{sec:training}).    
    \item Regarding \textsc{Inference} procedure: When \agentone{}'s~type is $\theta^{\textnormal{test}}$, can we infer $\hat{\theta}$ such that either $||\theta^{\textnormal{test}} - \hat{\theta}||$ is small, or \agentone's policies $\pi^x_{\hat{\theta}}$ and $\pi^x_{\theta^{\textnormal{test}}}$ are approximately equivalent? (see Section~\ref{sec:inference})
\end{enumerate}

\subsubsection{Smoothness properties}
For addressing Q.1, we introduce a number of properties characterizing our problem setting. These properties are essentially smoothness conditions on MDPs that enable us to make statements about the following intermediate issue: For two types $\theta, \theta'$, how ``similar" are the corresponding MDPs $\mathcal{M}(\theta), \mathcal{M}(\theta')$ from \agenttwo's point of view?


The first property characterizes the smoothness of rewards for \agenttwo~w.r.t.\ parameter $\theta$. Formally, the parametric MDP $\mathcal{M}(\theta)$ is $\alpha$-smooth with respect to the rewards if for any $\theta$ and $\theta'$ we have 
\begin{align}
  \max_{s \in S, a \in A} |R_\theta(s, a) - R_{\theta'}(s, a)| \leq \alpha \cdot \rmax \cdot ||\theta - \theta'||_2  
\end{align}

The second property characterizes the smoothness of policies for \agentone~w.r.t.\ parameter $\theta$; this in turn implies that the MDP's transition dynamics as perceived by \agenttwo~are smooth. Formally, the parametric MDP $\mathcal{M}(\theta)$ is $\beta$-smooth in the behavior of \agentone{} if for any $\theta$ and $\theta'$ we have 
\begin{align}
\max_{s \in S} \textrm{KL}{\big(\pi^x_{\theta}(.~|~s); \pi^x_{\theta'}(.~|~s) \big)} \leq \beta \cdot ||\theta - \theta'||_2.
\end{align}
For instance, one setting where this property holds naturally is when $\pi^x_{\theta}$ is a soft Bellman policy computed w.r.t.\ a reward function for \agentone{} which is smooth in $\theta$ \cite{ziebart10,kamalaruban2019interactive}.

    
The third property is a notion of influence as introduced by \cite{dimitrakakis2017multi}: This notion captures how much one agent can affect the probability distribution of the next state with her actions as perceived by the second agent. Formally, we capture the influence of \agentone~on \agenttwo~as follows:
\begin{align}
	\mathcal{I}_{x}\! := \! \max_{s \in S} \big(\max_{a, b,b'} \norm{T^{x,y}(.~|~s,a,b) - T^{x,y}(.~|~s,a,b')}_1\big), \label{eq.influence}
\end{align}
where $a$ represents the action of \agenttwo~, $b,b'$ represents two distinct actions of \agentone, and $T^{x,y}$ is the transition dynamics of the two-agent MDP (see Section~\ref{sec.setup.model}). Note that $\mathcal{I}_{x} \in [0, 1]$ and allows us to do fine-grained performance analysis: for instance, when $\mathcal{I}_{x} = 0$, then \agentone~doesn't affect the  transition dynamics as perceived by \agenttwo~and we can expect to have better performance for \agenttwo.



\subsubsection{Guarantees}
Putting this together, we can provide the following guarantees as an answer for Q.1:
\begin{theorem}
\label{theorem:approximateMDP}
Let $\theta^{\textnormal{test}} \in \Theta$ be the type of \agentone~at test time and \agenttwo~uses a policy $\pi^*_{\hat{\theta}}$ such that $||\theta^{\textnormal{test}} - \hat{\theta}|| \leq \epsilon$. The parameters $(\alpha, \beta ,\mathcal{I}_{x})$ characterize the smoothness as defined above. Then, the total reward achieved by \agenttwo~ satisfies the following guarantee
\begin{align*}
    J_{\theta^{\textnormal{test}}}(\pi^*_{\hat{\theta}}) \geq J_{\theta^{\textnormal{test}}}(\pi^*_{\theta^{\textnormal{test}}}) - \frac{\epsilon \cdot \alpha \cdot r_{\textnormal{max}}}{1-\gamma}
    - \frac{\mathcal{I}_{x} \cdot \sqrt{2 \cdot \beta \cdot \epsilon} \cdot r_{\textnormal{max}}}{(1-\gamma)^{2}}
\end{align*}
\end{theorem}
The proof of the theorem is provided in the supplementary material and builds up on the theory of approximate equivalence of MDPs by \cite{ApproximateEquival}. In the next two sections, we provide specific instantiations of \textsc{Training} and \textsc{Inference} procedures.

%% file: 5.1_algorithm1-Controller.tex
\begin{figure*}[t!]
\centering
\begin{subfigure}[b]{0.42\linewidth}
   \centering
   \includegraphics[width=0.90\textwidth]{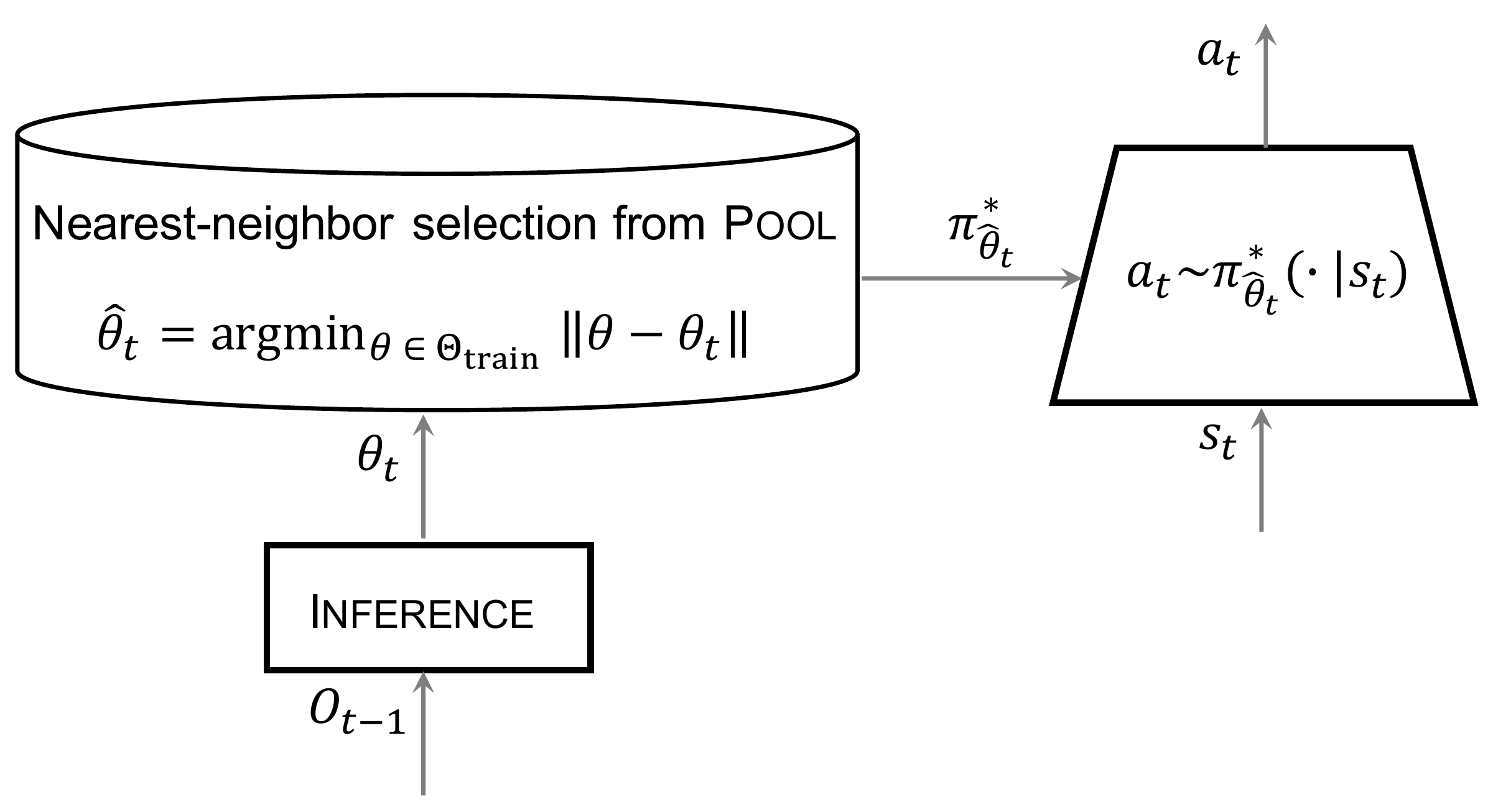}
   \caption{Test phase in Algorithm~\ref{alg:algorithm.model} with policy $\psi$ trained using \algcontroller procedure.}
   \label{fig:algos.controller} 
\end{subfigure}
\qquad
\qquad
\qquad
\begin{subfigure}[b]{0.42\linewidth}
   \centering
   \includegraphics[width=0.90\textwidth]{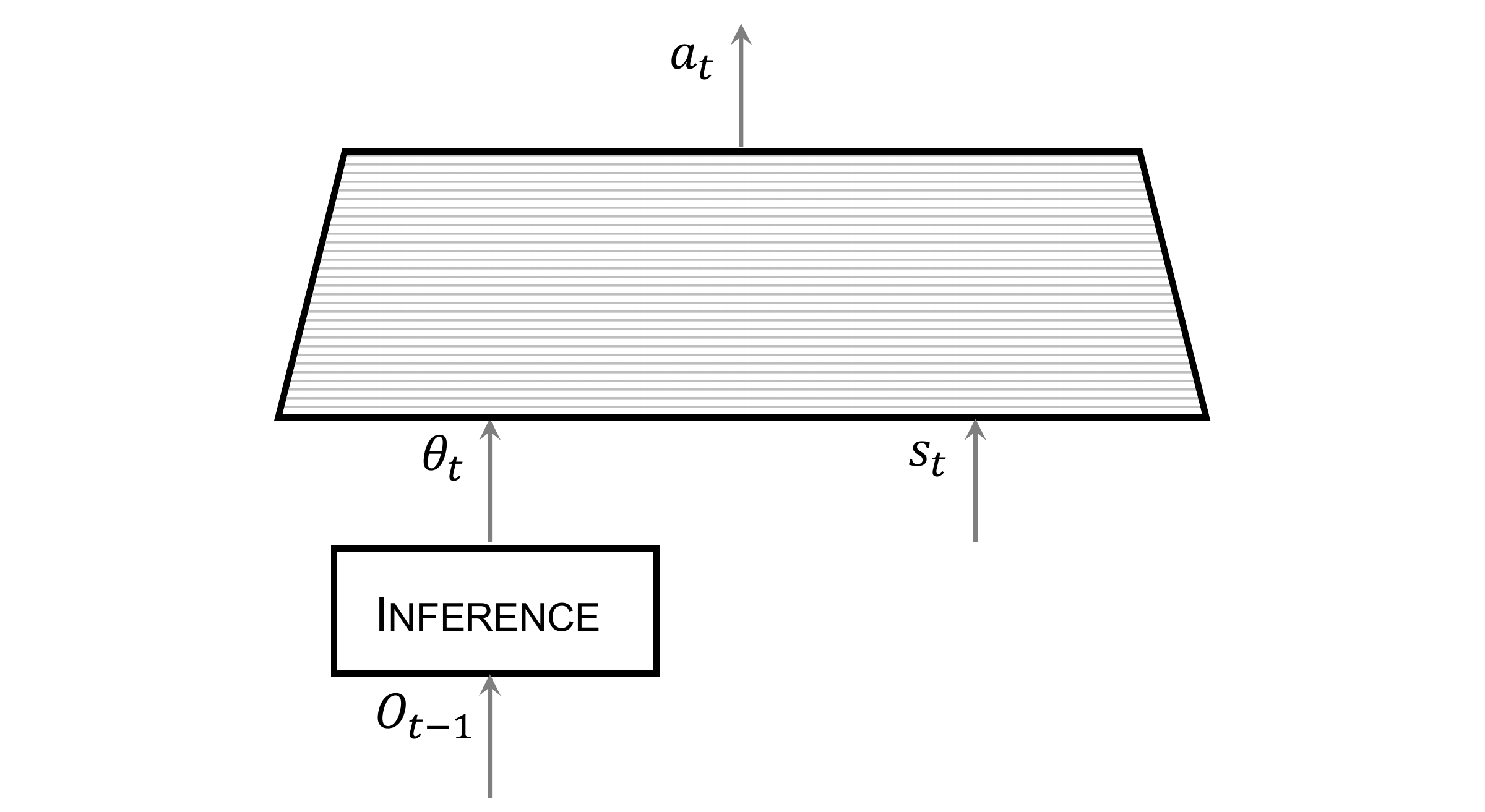}
   \caption{Test phase in Algorithm~\ref{alg:algorithm.model} with policy $\psi$ trained using \algdqn procedure.}
   \label{fig:algos.dqn}
\end{subfigure}
\caption{Two different instantiations of Algorithm~\ref{alg:algorithm.model} with the adaptive policy $\psi$ trained using procedures \algcontroller and \algdqn. \textbf{(a)} \algcontroller trains a set of best response policies $\{\pi^*_\theta ~|~ \theta \in \Theta_{\textnormal{train}} \}$. In the test phase at time step $t$ with $\theta_t$ as the output of \textsc{Inference}, the action $a_t$ is sampled from a distribution $\pi^*_{\hat{\theta}_{t}}(\cdot ~|~ s_t)$ where $\hat{\theta}_{t}$ is the nearest match for $\theta_t$ in the set $\Theta_{\textnormal{train}}$. \textbf{(b)} \algdqn trains one deep Q-Network (DQN) with an augmented state space given by $(s, \theta)$. At time $t$,  with $\theta_t$ as the output of \textsc{Inference}, the DQN network is given as  input a tuple $(s_t, \theta_t)$ and the network outputs an action $a_t$.} 
\label{fig:algos}
\end{figure*}


\section{\textsc{Training} Procedures}
\label{sec:training}
In this section, we present two procedures to train adaptive policies $\psi$ (see \textsc{Training} in Algorithm~\ref{alg:algorithm.model}).

\subsection{\textsc{Training} procedure \algcontroller}\label{sec.alg.controller}
The basic idea of \algcontroller is to maintain a pool $\pool$ of best response policies for \woagenttwo{} and, in the test phase, switch between these policies based on inference of the type $\thetatest$.

\subsubsection{Architecture of the policy $\psi$}
The adaptive pool based policy $\psi$ (\algcontroller) consists of a pool (\pool{}) of best response policies corresponding to different possible \agentone's{} types $\theta$, and a nearest-neighbor policy selection mechanism. 
In particular, when invoking \algcontroller{} for state $s_t$ and inferred \agentone's type $\theta_t$, the policy $\psi(s_t,\theta_t)$ first identifies the most similar \agentone{} in \pool, i.e., $\hat{\theta}_{t} = \arg \min_{\theta \in \Theta^{\textnormal{train}}} \| \theta - \theta_t \|$, and then executes an action $a_t \sim \pi^*_{\hat{\theta}_{t}}(\cdot ~|~ s_t)$ using the best response policy $\pi^*_{\hat{\theta}_{t}}$.


\subsubsection{Training process}
During training we compute a pool of best response policies \pool{} for a set of possible \agentone's types $\Theta^{\textnormal{train}}$, see Algorithm~\ref{alg:algorithm1.train}.

\begin{algorithm}[h!]
  \caption{\algcontroller: Training process}\label{alg:algorithm1.train}
  \begin{algorithmic}[1]
    \State \emph{Input:} Parameter space $\Theta^{\textnormal{train}}$
    \State $\pool \leftarrow \{\}$ 
    \For{each $\theta^{\textnormal{iter}} \in \Theta^{\textnormal{train}}$}
        \State $\pi^*_{\theta^{\textnormal{iter}}} \leftarrow$ best response policy for MDP $\mathcal{M}({\theta^{\textnormal{iter}}})$
        \State $\pool \leftarrow \pool \cup \{ (\theta^{\textnormal{iter}}, \pi^*_{\theta^{\textnormal{iter}}} )\}$
    \EndFor
    \State \Return $\pool$
  \end{algorithmic}
\end{algorithm}

\subsubsection{Guarantees}
It turns out that if the set of possible \agentone's types $\Theta^{\textnormal{train}}$ is chosen appropriately, Algorithm~\ref{alg:algorithm.model} instantiated with \algcontroller{} enjoys strong performance guarantees. 
In particular, choosing $\Theta^{\textnormal{train}}$ as a  sufficiently fine $\epsilon'$-cover of the parameter space $\Theta$, ensures that for any $\thetatest  \in \Theta$, that we might encounter at test time, we have considered a sufficiently \emph{similar} \agentone{} during training and hence can execute a best response policy which achieves good performance, see corollary below. 

\begin{corollary}
\label{corollary:pool}
Let $\Theta^{\textnormal{train}}$ be an $\epsilon'$-cover for $\Theta$, i.e., for all $\theta \in \Theta, \exists \theta' \in \Theta^{\textnormal{train}} \textnormal{ s.t. } ||\theta - \theta'|| \leq \epsilon'$. Let $\theta^{\textnormal{test}} \in \Theta$ be the type of \agentone~and the \textsc{Inference} procedure outputs $\theta_t$ such that $||\theta_t - \theta^{\textnormal{test}}|| \leq \epsilon''$. Let $\epsilon := \epsilon' + \epsilon''$. Then, at time $t$, the policy $\pi^*_{\hat{\theta}_{t}}$ used by \agenttwo~has the following guarantees:
\begin{align*}
    J_{\theta^{\textnormal{test}}}(\pi^*_{\hat{\theta}_t}) \geq J_{\theta^{\textnormal{test}}}(\pi^*_{\theta^{\textnormal{test}}}) - \frac{\epsilon \cdot \alpha \cdot r_{\textnormal{max}}}{1-\gamma}
    - \frac{\mathcal{I}_{x} \cdot \sqrt{2 \cdot \beta \cdot \epsilon} \cdot r_{\textnormal{max}}}{(1-\gamma)^{2}}
\end{align*}
\end{corollary}

Corollary~\ref{corollary:pool} follows from the result of Theorem~\ref{theorem:approximateMDP} given that the pool \pool{} of policies trained by \algcontroller is sufficiently rich. Note that the accuracy $\epsilon''$ of \textsc{Inference} would typically improve over time and hence the performance of the algorithm is expected to improve over time in practice, see Section~\ref{sec.expresults}.
%
Building on the idea of \algcontroller, next we provide a more practical implementation of training procedure which does not require to maintain an explicit pool of best response policies and therefore is easier to scale to larger problems.

%% file: 5.2_algorithm2-AugDQN.tex
\subsection{\textsc{Training} procedure \algdqn}\label{sec.alg.dqn}
\algdqn builds on the ideas of \algcontroller: Here, instead of explicitly maintaining a pool of best response policies for \agenttwo, we have a policy network trained on an augmented state space $S \times \Theta$.
This policy network resembles Deep Q-Network (DQN) architecture \cite{minh15}, but operates on an augmented state space and takes as input a tuple $(s, \theta)$. Similar architecture was used by \cite{hessel19multitask}, where one policy network was trained to play 57 Atari games, and the state space was augmented with the index of the game. In the test phase, \agenttwo~selects actions given by this policy network.

\subsubsection{Architecture of the policy $\psi$}
The adaptive policy $\psi$ (\algdqn) consists of a neural network trained on an augmented state space $S \times \Theta$. In particular, when invoking \algdqn for state $s_t$ and inferred \agentone's type $\theta_t$, we use the augmented state space $(s_t, \theta_t)$ as input to the neural network.  The output layer of the network computes the Q-values of all possible actions corresponding to the augmented input state. Agent $\mathcal{A}^{y}$ selects the action with the maximum Q-value.

\subsubsection{Training process} Here, we provide a description of how we train the policy network using augmented state space, see Algorithm~\ref{alg:algorithm2.train}. During one iteration of training the policy network, we first sample a parameter $\theta^\textnormal{iter} \sim \Theta^\textnormal{train}$. We then obtain the optimal best response policy $\pi^{*}_{\theta^\textnormal{iter}}$ of \agenttwo~for the MDP $\mathcal{M}(\theta^\textnormal{iter})$. We compute the \textit{vector} of all Q-values corresponding to this policy, i.e, $Q(s, a) \ \forall s \in S, a \in A$ (represented by $Q^{\pi^{*}_{\theta^\textnormal{iter}}}$ in Algorithm~\ref{alg:algorithm2.train}), using the standard Bellman equations \cite{suttonbarto1998}.
In our setting, we use these pre-computed Q-values to serve as the target values for the associated parameter $\theta^\textnormal{iter}$ for training the policy network. The loss function used for training is the standard squared error loss between the target Q-values computed using the procedure described above and those given by the network under training. The gradient of this loss function is used for back-propagation through the network. Multiple such iterations are carried out during training, until a convergence criteria is met. For more details on Deep Q-Networks, we refer the reader to see \cite{minh15}.

\begin{algorithm}[h!]
  \caption{\algdqn: Training process}\label{alg:algorithm2.train}
  \begin{algorithmic}[1]
    \State \emph{Input:} Parameter space $\Theta^{\textnormal{train}}$
    \State $\psi \leftarrow$ Init. policy network on augmented state space
    \While{\emph{convergence criteria is met}}
        \State sample $\theta^{\textnormal{iter}} \sim \textnormal{Uniform}(\Theta^{\textnormal{train}})$
        \State $\pi^*_{\theta^{\textnormal{iter}}} \leftarrow$ best response policy for MDP $\mathcal{M}({\theta^{\textnormal{iter}}})$
        \State $Q^{\pi^{*}_{\theta^{\textnormal{iter}}}} \leftarrow$ Q-values for policy $\pi^*_{\theta^{\textnormal{iter}}}$ in MDP $\mathcal{M}({\theta^{\textnormal{iter}}})$
         \State Train $\psi$ for one episode:
        \Statex \hspace*{3em} (i) by augmenting the state space with ${\theta^{\textnormal{iter}}}$
        \Statex \hspace*{3em} (ii) by using target Q-values $Q^{\pi^{*}_{\theta^{\textnormal{iter}}}}$
    \EndWhile
    \State \Return $\psi$
  \end{algorithmic}
\end{algorithm}

%% file: 5.3_inference.tex
\section{Inference Procedure}
\label{sec:inference}
In the test phase, the inference of \agentone{}'s type $\theta^\textnormal{test}$ from an observation history $O_{t-1}$ is a key component of our framework, and crucial for facilitating efficient collaboration.
Concretely, Theorem~\ref{theorem:approximateMDP} implies that a best response policy $\pi^*_{\hat{\theta}}$ also achieves good performance for \agentone{} with true parameter $\thetatest$ if $||\hat{\theta} - \thetatest||$ is small and MDP $\mdpM(\theta)$ is smooth w.r.t. parameter $\theta$ as described in Section~\ref{sec:performance-analysis}.



There are several different approaches that one can consider for inference, depending on application setting. For instance, we can use probabilistic approaches as proposed in the work of \cite{DBLP:conf/aaaiss/EverettR18} where a pool of \agentone's policies $\pi^x_\theta \ \forall \ \theta \in \Theta$ is maintained and inference is done at run time via simple probabilistic methods. Based on the work by \cite{grover18learning}, we can also maintain a more compact representation of \agentone's policies and then apply probabilistic methods on this representation.

We can also do inference based on ideas of inverse  reinforcement learning (IRL) where observation history $O_{t-1}$ serves the purpose of demonstrations \cite{abbeel2004apprenticeship,ziebart10}. This is particularly suitable when the parameter $\theta$ exactly corresponds to the rewards used by \agentone{} when computing its policy $\pi^x_\theta$. In fact, this is the approach that we follow for our inference module, and in particular, we employ the popular IRL algorithm, namely Maximum
Causal Entropy (MCE) IRL algorithm \cite{ziebart10}. We refer the reader to Section~\ref{sec.experiments.setup} for more details.

%% file: 6_experiments.tex
\begin{figure*}[t!]
\centering
\begin{subfigure}[b]{0.22\linewidth}
   \centering
   \includegraphics[width=0.9\textwidth]{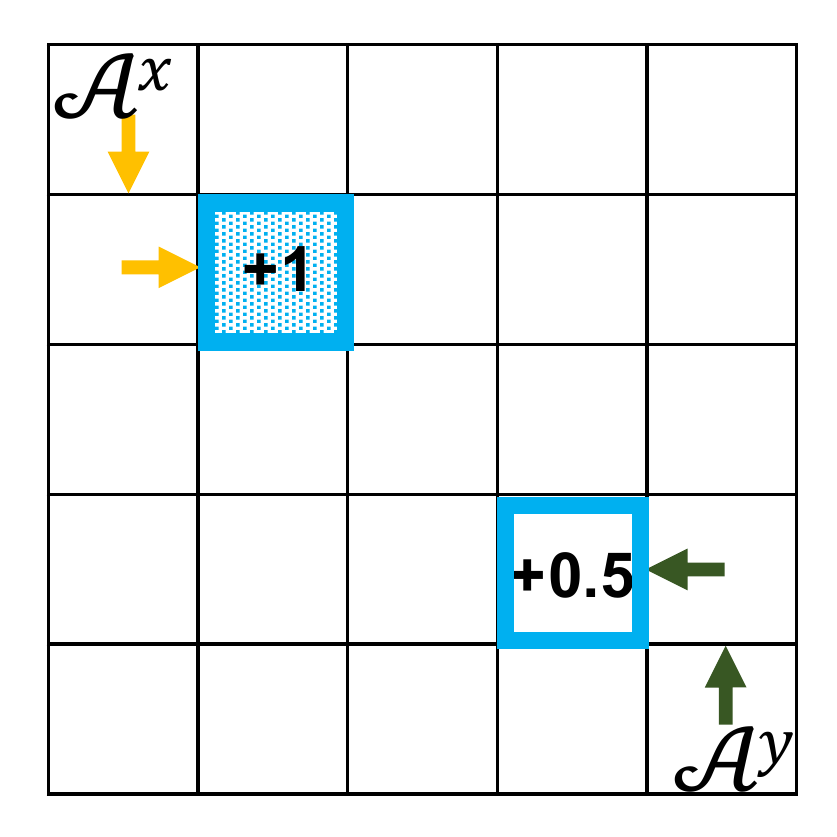}
   \caption{$\theta = [+1,+0.5]$}
   \label{fig:Gathering Game:1} 
\end{subfigure}
\quad
\begin{subfigure}[b]{0.22\linewidth}
   \centering
   \includegraphics[width=0.9\textwidth]{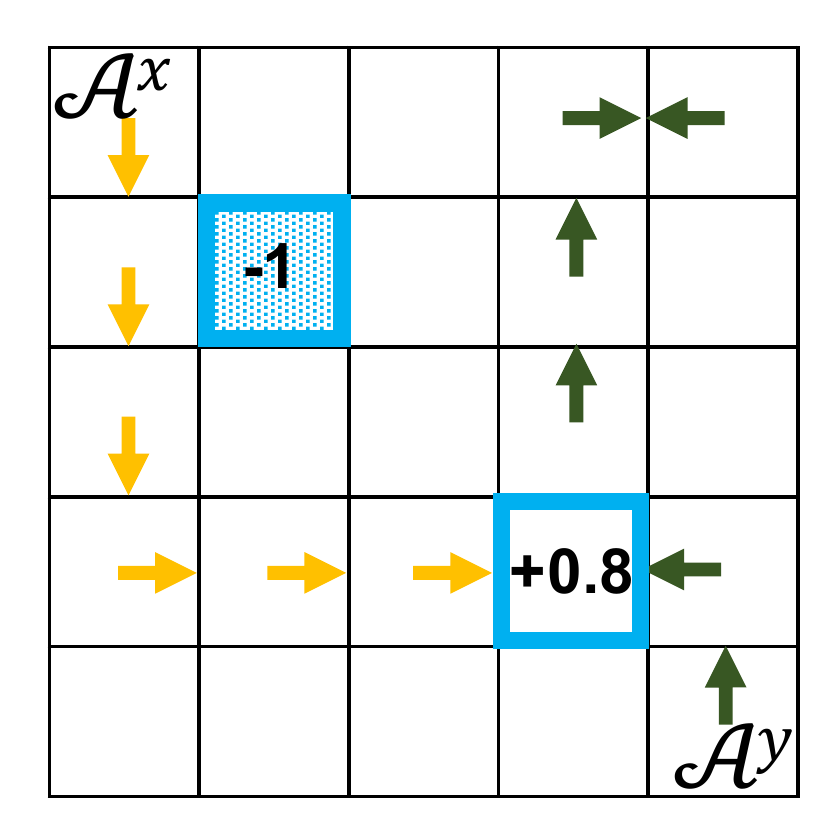}
   \caption{$\theta = [-1,+0.8]$}
   \label{fig:Gathering Game:2}
\end{subfigure}
\quad
\begin{subfigure}[b]{0.22\linewidth}
   \centering
	\includegraphics[width=0.9\textwidth]{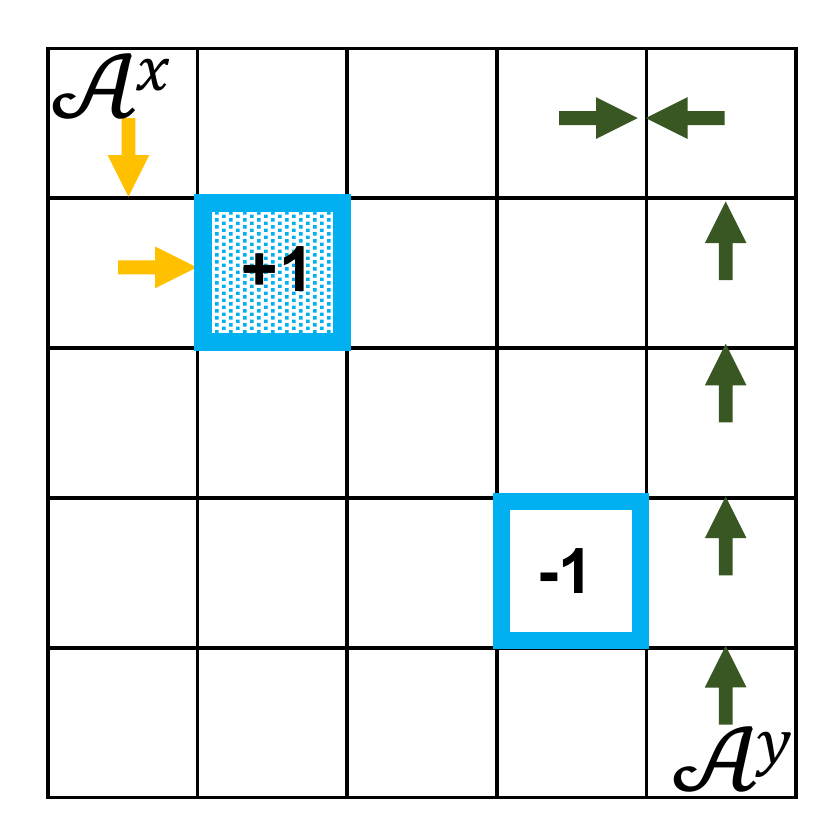}
	\caption{$\theta = [+1,-1]$}
	\label{fig:Gathering Game:3}
\end{subfigure}
\quad
\begin{subfigure}[b]{0.22\linewidth}
   \centering
	\includegraphics[width=0.9\textwidth]{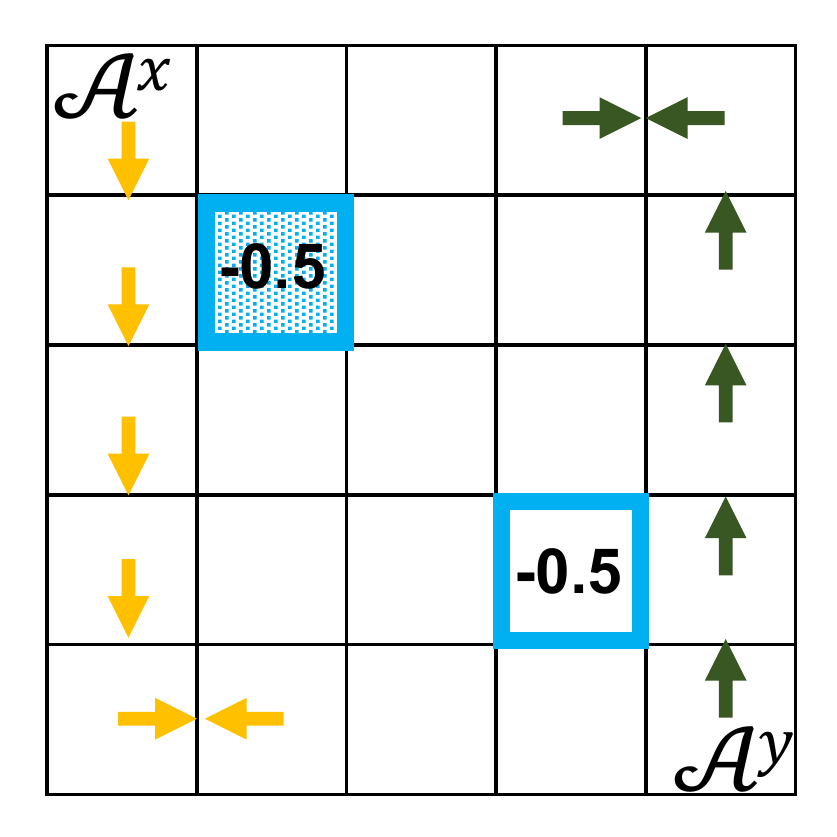}
	\caption{$\theta = [-0.5,-0.5]$}
	\label{fig:Gathering Game:4}
\end{subfigure}
\caption{We evaluate the performance on a gathering game environment, a variant of the environments considered by \cite{leibo17} and \cite{raileanu18modeling}. The objective is to maximize the total reward by collecting fruits while avoiding collisions and \agentone~is assisted by \agenttwo~in achieving this objective. The \textbf{environment} has a 5x5 grid space resulting in 25 grid cells and the \textbf{state space} is determined by the joint location of \agentone~and \agenttwo~(i.e., $|S| = 25 \times 25$). \textbf{Actions} are given by $A = $\textit{\{`step up', `step left', `step down', `step right', `stay'\}}. Each action is executed successfully with $0.8$ probability; with \textit{random move probability} of $0.2$, the agent is randomly placed in one of the four neighboring cells located in vertical or horizontal positions.
%
Two types of fruit objects are placed in two fixed grid cells (shown by `shaded blue' and `blue' cells). The rewards associated with these two fruit types are given by the parameter vector $\theta \in \Theta$ where $\Theta := [-1, +1]^2$. In our environment, the location of these two fruit types is fixed and fruits do not disappear (i.e., there is an unlimited supply of each fruit type in their respective locations).
For any fixed $\theta$, \textbf{\agentone's policy} $\pi_{\theta}^{x}$ is computed first by ignoring the presence of \agenttwo. 
From \agenttwo's point of view, each $\theta$ gives rise to a \textbf{parametric MDP} $\mdpM(\theta)$. \textbf{Transition dynamics} $T_{\theta}$ in $\mdpM(\theta)$ are obtained by marginalizing out the effect of \agentone's policy $\pi_{\theta}^{x}$. \textbf{Reward function} $R_{\theta}$ in $\mdpM(\theta)$ corresponds to the reward associated with fruits which depends on $\theta$; in addition to collecting fruits, \agenttwo~should avoid collision or close proximity to \agentone. This is enforced by a collision cost of $-5$ when \agenttwo~is in the same cell as \agentone, and a proximity cost of $-2$ when \agenttwo~is in one of the four neighboring cells located in vertical or horizontal positions.
The discount factor $\gamma$ is set to 0.99, and the initial state distribution $D_0$ corresponds to both agents starting in two corners.
%
The above four illustrations correspond to four different $\theta$ parameters, highlighting \agentone's policy $\pi^x_{\theta}$ and the best response policy $\pi^*_{\theta}$ for \agenttwo.
}
\label{fig:Gathering Game}
\end{figure*}

\section{Experiments}\label{sec.experiments}
We evaluate the performance of our algorithms using a gathering game environment, see Figure~\ref{fig:Gathering Game}. Below, we provide details of the experimental setup and then discuss results.

\subsection{Experimental setup}\label{sec.experiments.setup}
\subsubsection{Environment details}
For our experiments, we consider an episodic setting where two agents play the game repeatedly for multiple episodes enumerated as $e=1, 2, \ldots$. Each episode of the game lasts for 500 steps. Now, to translate the episode count to time steps $t$ as used in Algorithm~\ref{alg:algorithm.model} (line 3), we have $t = 500 \times e$ at the end of $e^{\textnormal{th}}$ episode.
%

For any fixed $\theta$, \agentone's policy $\pi_{\theta}^{x}$ is computed first by ignoring the presence of \agenttwo~as described below---this is in line with our motivating applications where \agentone~is the human-agent with a pre-specified policy.
%
In order to compute \agentone's policy $\pi_{\theta}^{x}$, we consider \agentone~operating in a single-agent MDP denoted  as $\mdpM^x(\theta) = (S^{x}, A, R^x_\theta, T^x, \gamma, D^x_0)$ where (i) $s \in S^{x}$ corresponds to the location of \agentone~in the grid-space, (ii) the action space is as described in Figure~\ref{fig:Gathering Game}, (iii) the reward function $R^x_{\theta}$ corresponds to reward associated with two fruit types given by $\theta$, (iv) $T^x$ corresponds to transition dynamics of \agentone~alone in the environment, (v) discount factor $\gamma = 0.99$, and (vi) $D^x_0$ corresponds to \agentone~starting in the upper-left corner (see Figure~\ref{fig:Gathering Game}). Given $\mdpM^x(\theta)$, we compute $\pi_{\theta}^{x}$ as a soft Bellman policy -- suitable to capture sub-optimal human behaviour in applications \cite{ziebart10}.

\looseness-1
From \agenttwo's point of view, each $\theta$ gives rise to a parametric MDP $\mdpM(\theta)$ in which \agenttwo~is operating in the game along with the corresponding \agentone, see Figure~\ref{fig:Gathering Game}. 
\subsubsection{Baselines and implementation details.}
We use three baselines to compare the performance of our algorithms: (i) $\textsc{Rand}$ corresponds to picking a random $\theta \in \Theta$ and using best response policy $\pi^*_{\theta}$, (ii) $\textsc{FixedMM}$ corresponds to the fixed best response (in a minmax sense) policy in Eq.~\ref{maxmin_formulation_static}, and (iii) $\textsc{FixedBest}$ is a variant of $\textsc{FixedMM}$ and corresponds to the fixed best response (in a average sense) policy.



We implemented two variants of \algcontroller which store policies corresponding to $\epsilon'=1$ and $\epsilon' = 0.25$ covers of $\Theta$ (see Corollary~\ref{corollary:pool}), denoted as $\textsc{AdaptPool}_{1}$ and $\textsc{AdaptPool}_{0.25}$ in Figure~\ref{fig:PerfomanceMeasures}.
%
%
%
%
%
\looseness-1
Next, we give specifications of the trained policy network used in \algdqn. We used $\Theta^{\textnormal{train}}$ to be a $0.25$ level discretization of $\Theta$. 
The trained network $\psi$ has 3 hidden layers with leaky RELU-units (with $\alpha= 0.1$) having $64$, $32$, and $16$ hidden units respectively, and a linear output layer with $5$ units (corresponding to the size of action set $|A|$) (see \cite{minh15} for more details on training Deep Q-Network). The input to the neural network is a concatenation of the location of the $2$ agents, and the parameter vector $\theta_t$, where $|\theta_t| = 2$ (this corresponds to the augmented state space described in Section~\ref{sec.alg.dqn}). The location of each agent is represented as a one-hot encoding of a vector of length $25$ corresponding to the number of grid cells
Hence the length of the input vector to the neural network is $25 \times2 + 2~(=52)$. During training, \agenttwo~implemented epsilon-greedy exploratory policies (with exploration rate decaying linearly over training iterations from 1.0 to 0.01). Training lasted for about 50 million iterations.

\looseness-1
Our inference module is based on the MCE-IRL approach \cite{ziebart10} to infer $\theta^{\textnormal{test}}$ by observing actions taken by \agentone's policy. Note that, we are using MCE-IRL to infer the reward function parameters $\theta^{\textnormal{test}}$ used by \agentone~for computing its policy in the MDP $\mdpM^x(\theta^{\textnormal{test}})$ (see ``Environment details" above). At the beginning, the inference module is initialized with $\theta_0 = [0,0]$, and its output at time $t$ given by $\theta_t$ is based on history $O_{t-1}$. In particular, we implemented a sequential variant of MCE-IRL algorithm which updates the estimate $\theta_t$ only at the end of every episode $e$ using stochastic gradient descent  with learning rate $\eta = 0.001$. We refer the reader to \cite{ziebart10} for details on the original MCE-IRL algorithm and to \cite{kamalaruban2019interactive} for the sequential variant.


\subsection{Results: Worst-case and average performance}
\label{sec.expresults}
We evaluate the performance of algorithms on $441$ different $\theta^\textnormal{test}$ obtained by a $0.1$ level discretization of the 2-D parametric space $\Theta := [-1, +1]^2$. For a given $\theta^\textnormal{test}$, the results were averaged over $10$ runs. 
Results are shown in Figure~\ref{fig:PerfomanceMeasures}. As can be seen in Figure~\ref{fig:OptVsAlg_worse}, the worst-case performance of both \algdqn and \algcontroller is significantly better than that of the three baselines ($\textsc{FixedBest}$, $\textsc{Rand}$ and $\textsc{FixedMM}$), indicating robustness of our algorithmic framework.  
In our experiments, the $\textsc{FixedMM}$ and $\textsc{FixedBest}$ baselines correspond to best response policies $\pi^{*}_{\theta}$ for $\theta=[0.1, -1]$ and $\theta = [0,-0.1]$ respectively. Under both these policies, \agenttwo's behavior is qualitatively similar to the one shown in Figure~\ref{fig:Gathering Game:3}. As can be seen, under these policies, \agenttwo~avoids both fruits and avoids any collision; however, this does not allow \agenttwo~to assist \agentone~in collecting fruits even in scenarios where fruits have positive rewards.

In Figure~\ref{fig:irl_module}, we show the convergence behavior of the inference module. Here, \textsc{Worst} shows the worst case performance: As can be seen in the \textsc{Worst} line, there are cases where the performance of the inference procedure is bad, i.e., $\norm{\theta_t - \theta^{\textnormal{test}}}$ is large. This usually happens when different parameter values of $\theta$ results in \agentone~having equivalent policies. In these cases, estimating the exact  $\theta^{\textnormal{test}}$  without any additional information is difficult. In our experiments, we noted that even if $\norm{\theta_t - \theta^{\textnormal{test}}}$ is large, it is often the case that  \agentone's policies $\pi^x_{\theta_t}$ and $\pi^x_{\theta^{\textnormal{test}}}$ are approximately equivalent which is important for getting a good approximation of the transition dynamics $T_{\theta^{\textnormal{test}}}$. Despite the poor performance of the inference module in such cases, the performance of our algorithms is significantly better than the baselines (as is evident in Figure~\ref{fig:OptVsAlg_worse}). In the supplementary material, we provide additional experimental results corresponding to the algorithms' performance for each individual $\theta^{\textnormal{test}}$ to gain further insights.

\begin{figure*}[t!]
\centering
\begin{subfigure}[b]{0.32\linewidth}
   \centering
   \includegraphics[width=0.97\textwidth]{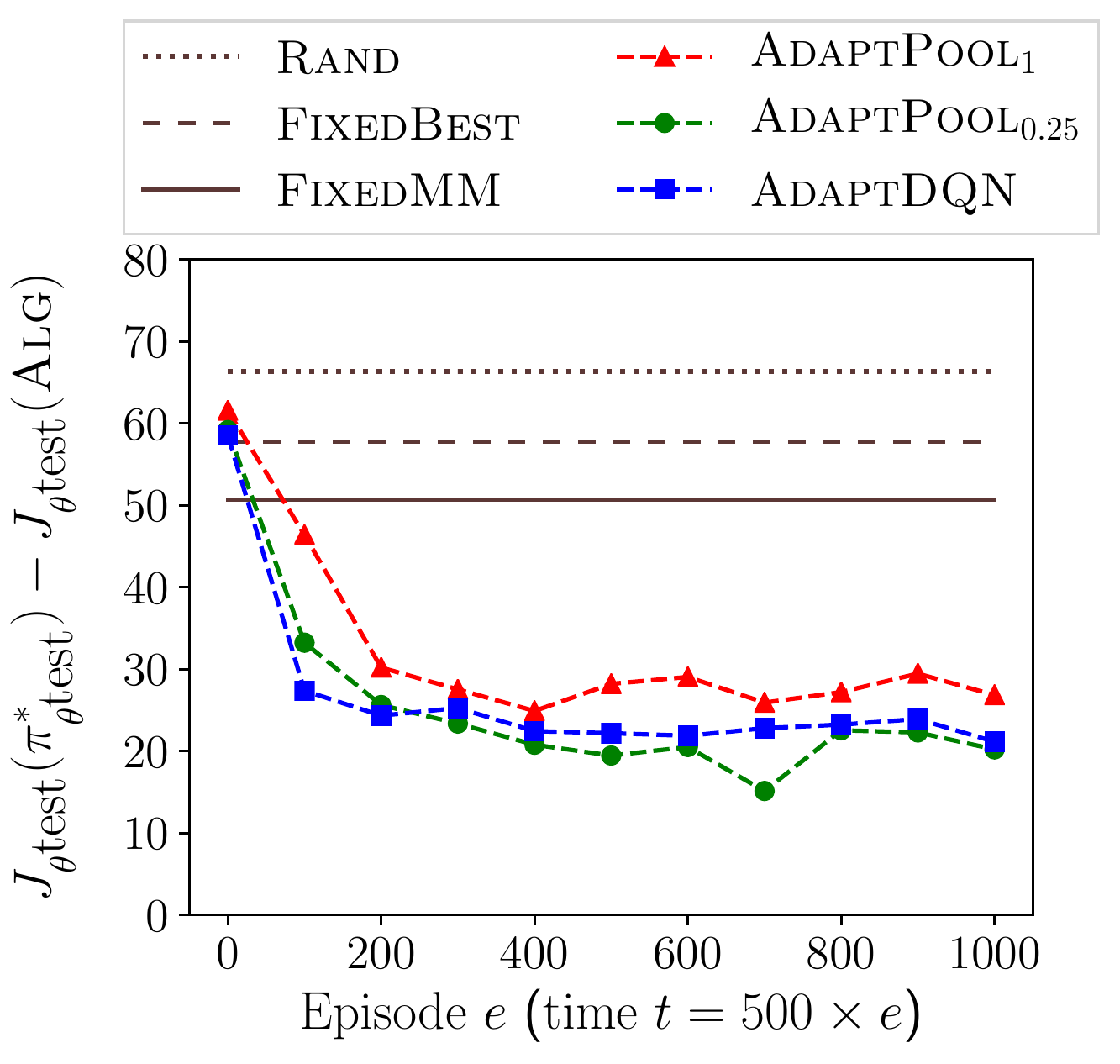}
   \caption{Total reward: Worst-case}
   \label{fig:OptVsAlg_worse}
\end{subfigure}
\begin{subfigure}[b]{0.32\linewidth}
   \centering
   \includegraphics[width=1.0\textwidth]{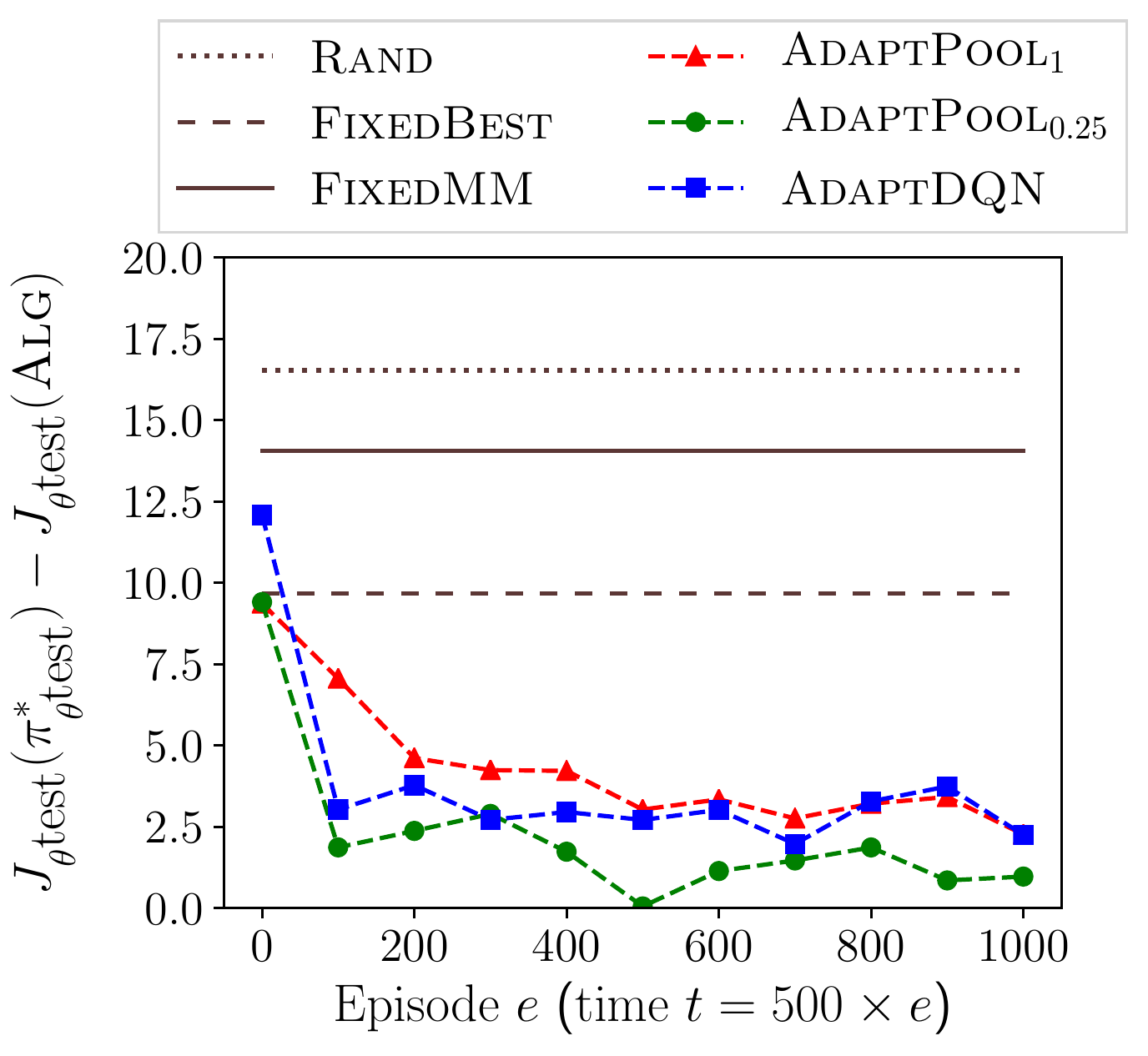}
   \caption{Total reward: Average-case}
   \label{fig:OptVsAlg_avg} 
\end{subfigure}
   \centering
   \begin{subfigure}[b]{0.32\linewidth}
   \includegraphics[width=0.94\textwidth]{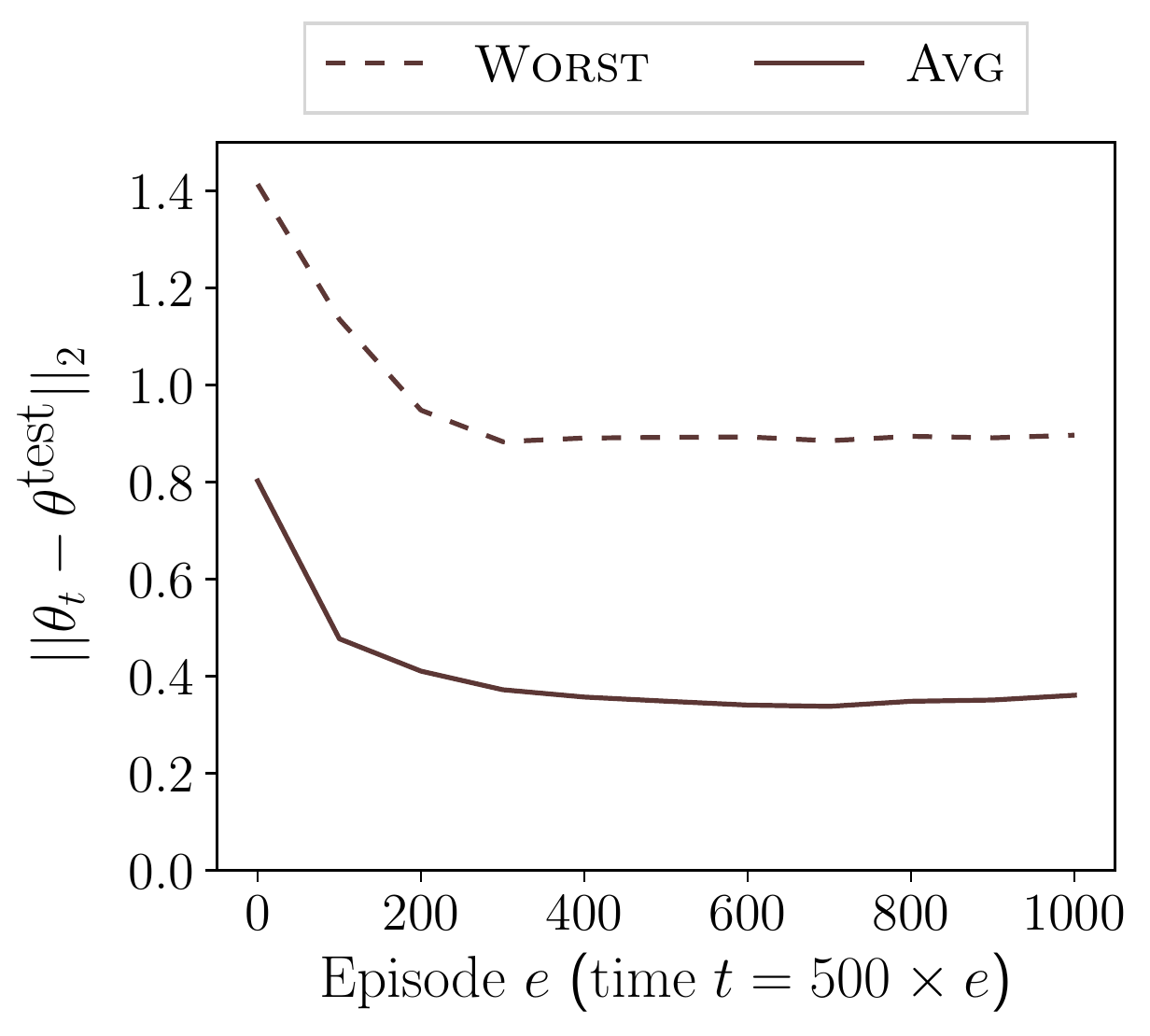}
   \caption{Inference module}
   \label{fig:irl_module}
 \end{subfigure}
\caption{(a) Worst-case performance of both \algdqn and \algcontroller is significantly better than that of the baselines, indicating robustness of our algorithmic framework. (a, b) Two variants of \algcontroller are shown corresponding to  $1$-cover and $0.25$-cover. As expected, the algorithm $\algcontroller_{0.25}$ with larger pool size has better performance compared to the algorithm $\algcontroller_{1}$. (c) Plot shows the convergence behavior of the inference module as more observational data is gathered: \textsc{Avg} shows the average performance (averaged $\norm{\theta_t - \theta^{\textnormal{test}}}$ w.r.t. different $\thetatest$) and \textsc{Worst} shows the worst case performance (maximum   $\norm{\theta_t - \theta^{\textnormal{test}}}$ w.r.t. different $\thetatest$).
} 
    \vspace{-2mm}
\label{fig:PerfomanceMeasures}
\end{figure*}

%% file: 8.0_appendix-aaai-experiments.tex
\vspace{-6mm}
\subsection{Results: Performance heatmaps for each $\theta^{\textnormal{test}}$}
Here, we provide additional experimental results to gain further insights into the performance of our algorithms. These results are presented in Figure~\ref{fig:heatmaps} in the form of heat maps for each individual $\theta^{\textnormal{test}}$: Heat maps either represent performance of algorithms (in terms of the total reward $J_{\thetatest}(\textsc{Alg})$) or the performance of inference procedure (in terms of the norm $\norm{\theta_t - \theta^{\textnormal{test}}}$). These results are plotted in the episode $e=1000$ (cf., Figure~\ref{fig:PerfomanceMeasures} where the performance was plotted over time with increasing $e$).

\begin{figure*}[t!]
	\centering
	\begin{subfigure}[b]{0.19\textwidth}
			\centering
			\includegraphics[width=1\linewidth]{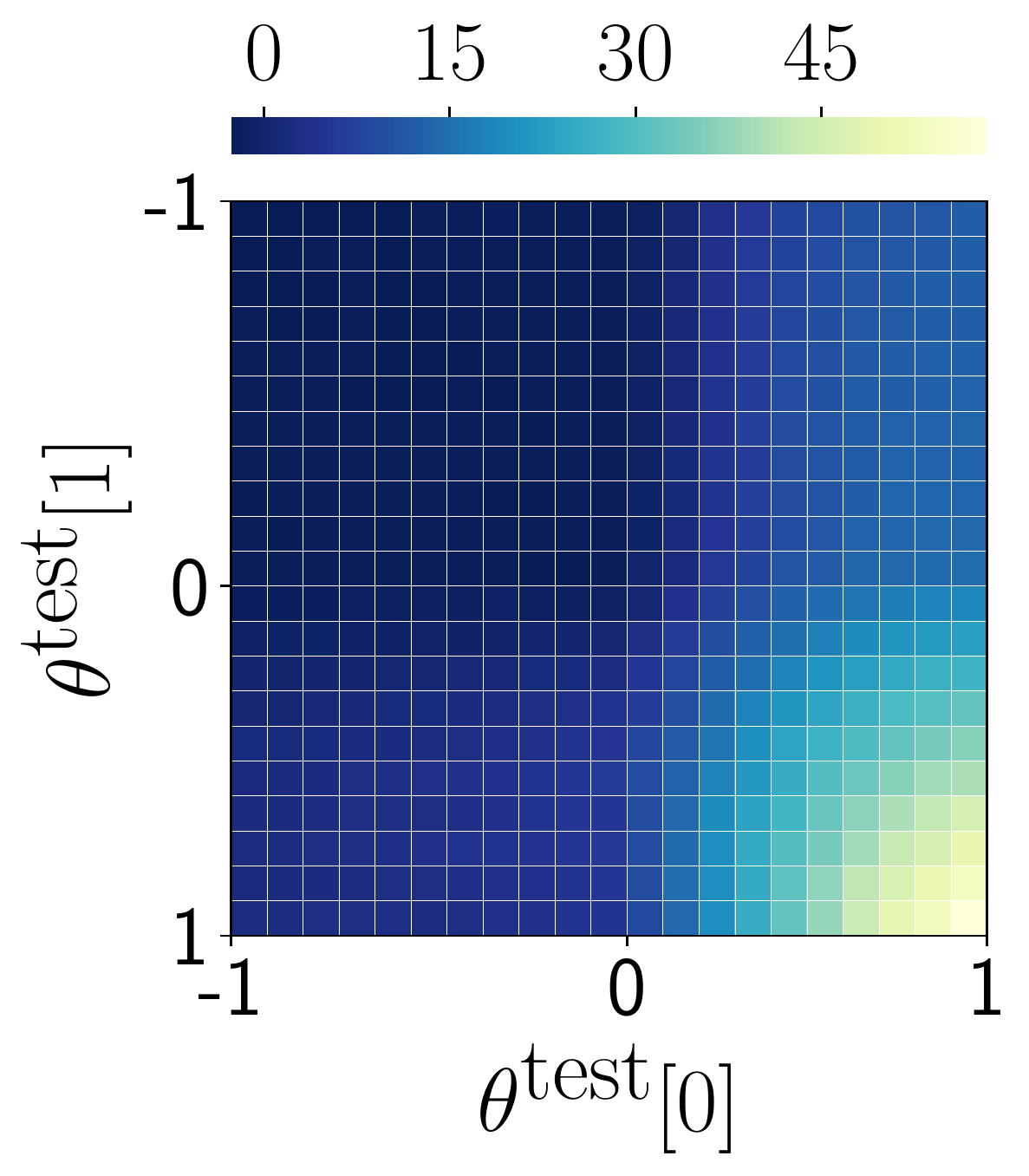}
			\caption{$\pi_{\theta^\textnormal{test}}^{*}$}
			\label{fig:opt}
	\end{subfigure}
	\ 
	\begin{subfigure}[b]{0.19\textwidth}
			\centering
			\includegraphics[width=1\linewidth]{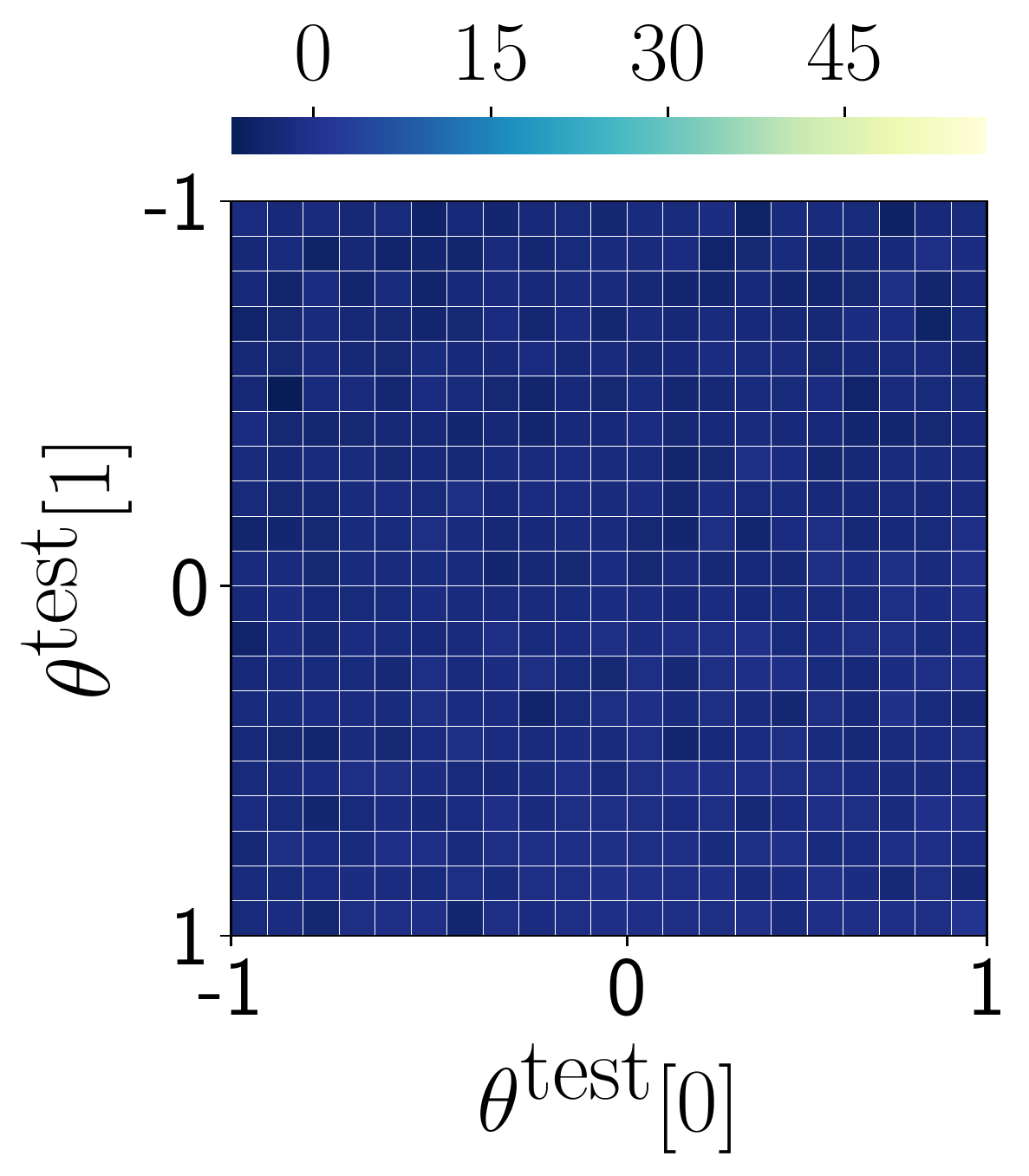}
			\caption{$\textsc{FixedBest}$}
			\label{fig:fixedbest}
	\end{subfigure}
	\ 
	\begin{subfigure}[b]{0.19\textwidth}
			\centering
			\includegraphics[width=1\linewidth]{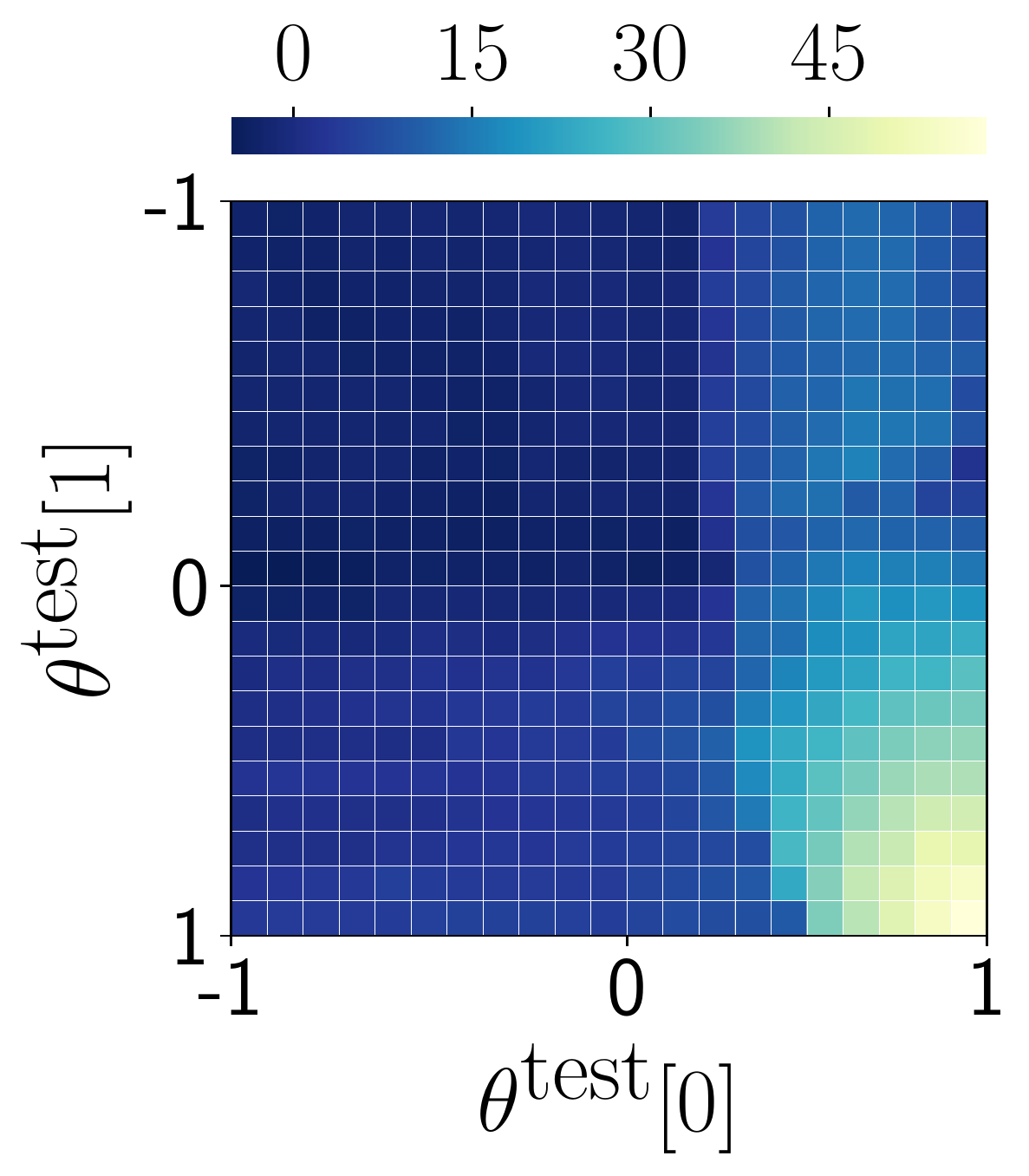}
			\caption{$\textsc{AdaptPool}_{0.25}$}
			\label{fig:pool_0.25}
	\end{subfigure}
	\ 
	\begin{subfigure}[b]{0.19\textwidth}
			\centering
			\includegraphics[width=1\linewidth]{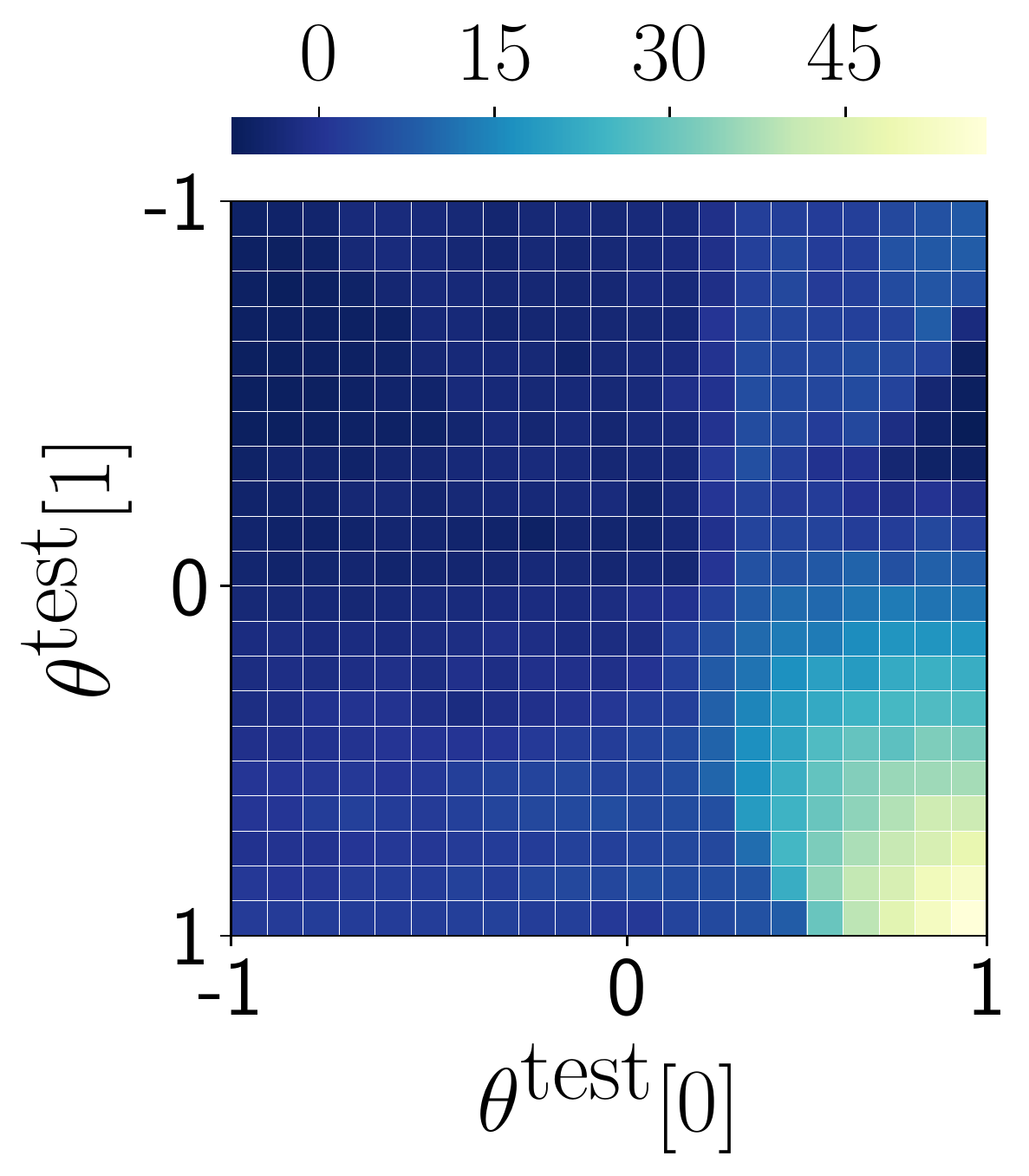}
			\caption{$\textsc{AdaptDQN}$}
			\label{fig:adqn}
	\end{subfigure}
	\ 
	\begin{subfigure}[b]{0.19\textwidth}
			\centering
			\includegraphics[width=1\linewidth]{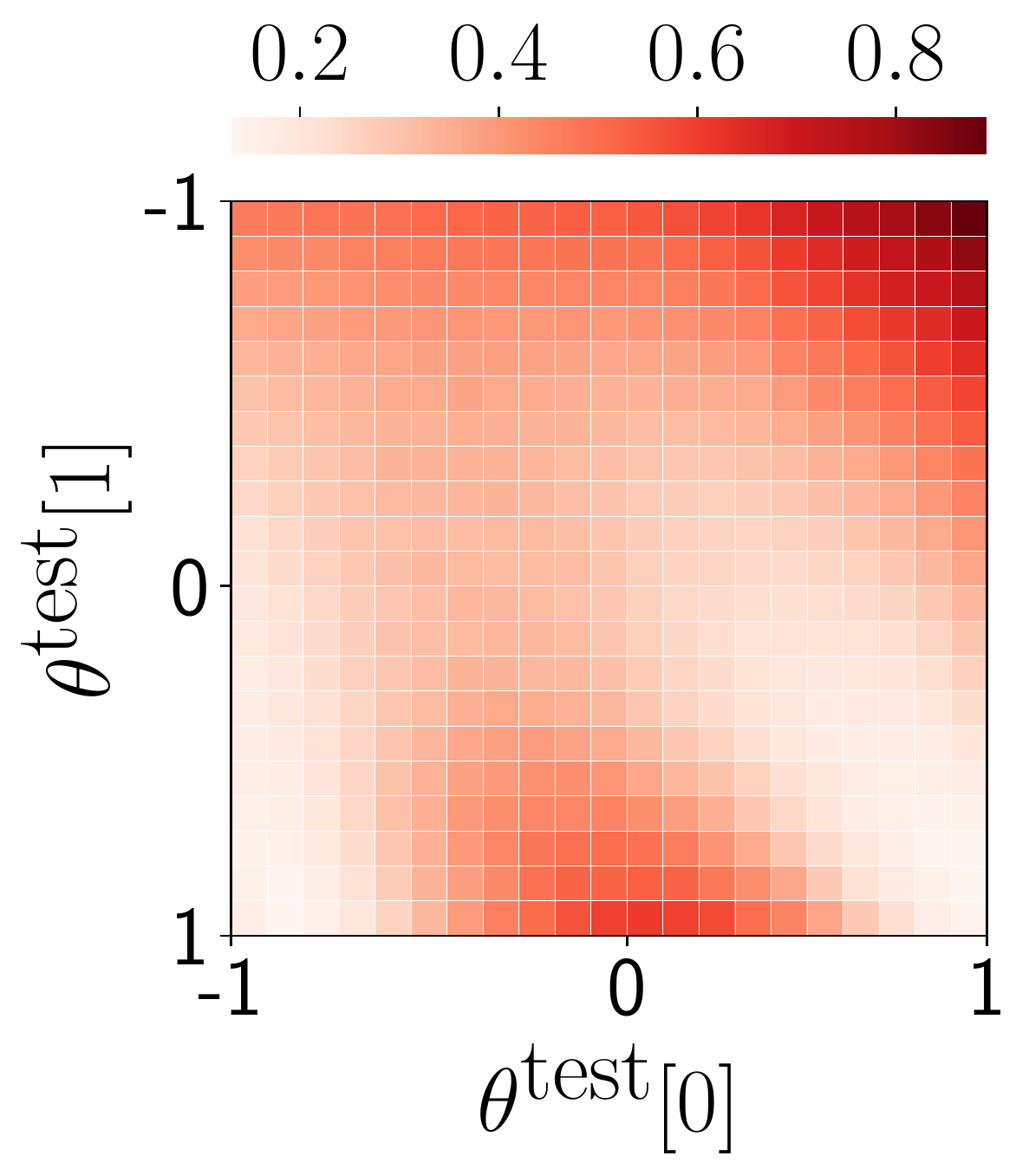}
			\caption{Inference Module}
			\label{fig:inf_module}
	\end{subfigure}
    \vspace{-1mm}	
	\caption{(a, b, c, d) Heat map of the total rewards obtained by different algorithms when measured in the episode $e=1000$. (e) Heat map of the norm $\norm{\theta_t - \theta^{\textnormal{test}}}$, i.e., the gap between the estimated and true parameter $\theta^{\textnormal{test}}$ at the end of episode $e=1000$. The performance of the inference procedure is poor in cases when different parameter values of $\theta^{\textnormal{test}}$ results in \agentone~having equivalent policies. However, in these cases as well, the performance of our algorithms ($\algcontroller_{0.25}$ and $\algdqn$ are shown in the figure) is significantly better than the baselines (\textsc{FixedBest} is shown in the figure).}
	\label{fig:heatmaps}
    \vspace{-4mm}
\end{figure*}

It is important to note that there are cases where the performance of inference procedure is bad, i.e., $\norm{\theta_t - \theta^{\textnormal{test}}}$ is large. This usually happens when different parameter values of $\theta^{\textnormal{test}}$ results in \agentone~having equivalent policies. In these cases, estimating the exact  $\theta^{\textnormal{test}}$ without any additional information is difficult. In our experiments, we noted that even if $\norm{\theta_t - \theta^{\textnormal{test}}}$ is large, it is often the case that \agentone's policies $\pi^x_{\theta_t}$ and $\pi^x_{\theta^{\textnormal{test}}}$ are approximately equivalent which is important for getting a good approximation of the transition dynamics $T_{\theta^{\textnormal{test}}}$. Despite the poor performance of the inference module in such cases, the performance of our algorithms (see $\algcontroller_{0.25}$ and $\algdqn$ in the figure) is significantly better than the baselines (see \textsc{FixedBest} in the figure).

%% file: 2_relatedwork.tex
\vspace{-2mm}
\section{Related Work}\label{sec.relatedwork}
\vspace{-2mm}

\textbf{Modeling and inferring about other agents.}
The inference problem has been considered in the literature in various forms. For instance, \cite{grover18learning} consider the problem of learning policy representations that can be used for interacting with unseen agents when using representation-conditional policies. They also consider the case of inferring another agent's representation (parameters) during test time.
%
\cite{macindoe2012pomcop} consider planners for collaborative domains that can take actions to learn about the intent of another agent or hedge against its uncertainty.
\cite{DBLP:conf/hri/NikolaidisRGS15} cluster human users into types and aim to  infer the type of new user online, with the goal of executing the policy for that type. 
They test their approach in robot-human interaction but do not provide any theoretical analysis of their approach.
%
Beyond reinforcement learning, the problem of modeling and inferring about other agents has been studied in other applications such as personalization of web search ranking results by inferring user's preferences based on their online activity~\cite{white2013enhancing,white2014devices,singla2014enhancing}.

\textbf{Multi-task and meta-learning.}
Our problem setting can be interpreted as a multi-task RL problem in which each possible \agentone{} corresponds to a different task, or as a meta-learning RL problem in which the goal is to learn a policy that can quickly adapt to new partners.
\cite{hessel19multitask} study the problem of multi-task learning in the RL setting in which a single agent has to solve multiple tasks, e.g., solve all Atari games. However, they do not consider a separate test set to measure generalization of trained agents but rather train and evaluate on the same tasks.
\cite{smundsson2018meta} consider the problem of meta learning for RL in the context of changing dynamics of the environment and approach it using a Gaussian processes and a hierarchical latent variable model approach.

\textbf{Robust RL.}
The idea of robust RL is to learn policies that are robust to certain types of errors or mismatches.
In the context of our paper, mismatch occurs in the sense of encountering human agents that  have not been encountered at training time and the learned policies should be robust in this situation.
\cite{DBLP:conf/icml/PintoDSG17} consider training of policies in the context of a \emph{destabilizing adversary} with the goal of coping with model mismatch and data scarcity.
\cite{DBLP:conf/nips/RoyXP17} study the problem of RL under model mismatch such that the learning agent cannot interact with the actual test environment but only a reasonably close approximation.
The authors develop robust model-free learning algorithms for this setting. 
%

\textbf{More complex interactions, teaching, and steering.} In our paper, the type of interaction between two agents is limited as \agenttwo{} does not affect \agentone{}'s behaviour, allowing us to gain a deeper theoretical understanding of this setting.
There is also a related literature on ``steering'' the behavior of other agent. For example, (i) the \textit{environment design} framework of~\cite{zhang2009policy}, where  one agent tries to steer the behavior of another agent by modifying its reward function, (ii) the \textit{cooperative inverse reinforcement learning} of~\cite{hadfield-menell16cooperative}, where the human uses demonstrations to reveal a proper reward function to the AI agent, and (iii) the \textit{advice-based interaction} model~\cite{amir2016interactive},  where the goal is to communicate advice to a sub-optimal agent on how to act.

\looseness-1
\textbf{Dealing with non-stationary agents.}
The work of \cite{DBLP:conf/aaaiss/EverettR18} is closely related to ours: they design a \emph{Switching Agent Model} (SAM) that combines deep reinforcement learning with opponent modelling to robustly switch between multiple policies. \cite{DBLP:conf/nips/ZhengMHZYF18} also consider a similar setting of detecting non-stationarity and reusing policies on the fly, and introduce \textit{distilled policy network} that serves as the policy library. Our algorithmic framework is similar in spirit to these two papers, however, in our setting, the focus is on acting optimally against an unknown agent whose behavior is stationary and we provide theoretical guarantees on the performance of our algorithms. 
%
\cite{DBLP:conf/aaai/SinglaH018} have considered the problem of learning with experts advice where experts are not stationary and are learning agents themselves. However, their focus is on designing a meta-algorithm on how to coordinate with these experts and is technically very different from ours.
%
A few other recent papers have also considered repeated human-AI interaction where the human agent is non-stationary and is evolving its behavior in response to AI agent (see \cite{radanovic2019learning,nikolaidis2017game}. Prior work also considers a learner that is aware of the presence of other actors (see \cite{DBLP:conf/atal/FoersterCAWAM18,raileanu18modeling}).

%% file: 7_conclusions.tex
\vspace{-3.5mm}
\section{Conclusions}
\vspace{-2.5mm}
\looseness-1 Inspired by real-world applications like virtual personal assistants, we studied the problem of designing AI agents that can robustly cooperate with new people in human-machine partnerships.
Inspired by our motivating applications, we focused on an important practical aspect that there is often a clear distinction between the training and test phase: the explicit reward information is only available during training but adaptation is also needed during testing.
We provided a framework for designing adaptive policies and gave theoretical insights into its robustness. In experiments, we demonstrated that these policies can achieve good performance when interacting with previously unseen agents.

%% file: 8.1_appendix-aaai-theorem1.tex
\section{Proof of Theorem~\ref{thm_badperformance}}
\begin{proof}
We provide a proof via constructing a problem instance. Let the parametric space be $\Theta = \{\theta_1, \theta_2\}$. Next, we define two MDPs $\mathcal{M}(\theta_1) := (S, A, T_{\theta_1}, R_{\theta_1}, \gamma, \mathcal{D}_0)$  and $\mathcal{M}(\theta_2) := (S, A, T_{\theta_2}, R_{\theta_2}, \gamma, \mathcal{D}_0)$ below:
\begin{itemize} 
\item set of states is given by $S = \{gold, end\}$, with $s$ denoting a generic state. Here, state $gold$ represents a state where reward can be accumulated, and state $end$ is a terminal state.
\item set of actions is given by $A = \{a_1, a_2\}$  with $a \in A$ denoting a generic action for \agenttwo~and $a^x \in A$ denoting a generic action for \agentone. 
\item for $\theta \in \Theta$, we have $R_{\theta}(s, a) = 0$ if $s = end$ and $R_{\theta}(s, a) = \rmax$ if $s=gold$ where $\rmax > 0$. Note that the reward function only depends on the state and not on the action taken. Also, the reward function is same for both $\theta_1$ and $\theta_2$.
\item discount factor $\gamma \in [0, 1)$ and initial state distribution $\mathcal{D}_0$ is given by $\mathcal{D}_0(s_0 = gold) = 1$.
\item most crucial part of this problem instance is the transition dynamics $T_{\theta_1}$ and $T_{\theta_2}$ that we specify below. Note that, for $\theta \in \Theta$, $T_{\theta}(s'~|~s, a)=\E_{a^x}[T^{x,y}(s'~|~s, a, a^x)]$, where $a^x \sim \pi^x_{\theta}(\cdot~|~s)$, i.e., $T_{\theta}(s'~|~s, a)$ corresponds to the transition dynamics derived from a two agent MDP for which \agentone{}'s policy is $\pi^x_{\theta}$. We define transition dynamics $T^{x,y}(s'~|~s, a, a^x)$ below in Figure~\ref{appendix.theorem1.fig-dynamics}, and policies $\pi^x_{\theta_1}$ and $\pi^x_{\theta_2}$ below in Figure~\ref{appendix.theorem1.fig-policy}.
\end{itemize}

\renewcommand\thesubfigure{\roman{subfigure}}
\begin{figure*}[!h]
	\centering
	\begin{subfigure}[b]{0.45\textwidth}
		\centering
		\begin{tabular}[h!]{c|c|c|}
			\diagbox{\woagenttwo's action}{\woagentone's action} & $a^x = a_1$ & $a^x = a_2$ \\ \hline
			$a = a_1$   & $1$ & $0$ \\ \hline
			$a = a_2$  & $0$ & $1$  \\ \hline
		\end{tabular}
		\caption{$T^{x,y}(s'=gold~|~s=gold, a, a^x)$}
  	\end{subfigure}
	\hfill
	\begin{subfigure}[b]{0.45\textwidth}
		\centering
		\begin{tabular}[h!]{c|c|c|}
			\diagbox{\woagenttwo's action}{\woagentone's action} & $a^x = a_1$ & $a^x = a_2$ \\ \hline
			$a = a_1$   & $0$ & $0$ \\ \hline
			$a = a_2$  & $0$ & $0$  \\ \hline
		\end{tabular}
		\caption{$T^{x,y}(s'=gold~|~s=end, a, a^x)$}		
  \end{subfigure}
  \caption{Transition dynamics of a two agent MDP with $a \in A$ denoting a generic action for \agenttwo~and $a^x \in A$ denoting a generic action for \agentone. (i) From state $s=gold$, if  both agents take the same action, the next state is $s=gold$, otherwise the next state is $s=end$. (ii) From state $s=end$, any pair of actions results in the next state as $s=end$.}		
  \label{appendix.theorem1.fig-dynamics}
\end{figure*}

\begin{figure*}[!h]
	\centering
	\begin{subfigure}[b]{0.45\textwidth}
		\centering
		\begin{tabular}[h!]{c|c|c|}
			\diagbox{state}{action} & $a^x = a_1$ & $a^x = a_2$ \\ \hline
			$s = gold$   & $1$ & $0$ \\ \hline
			$s = end$  & $0.5$ & $0.5$  \\ \hline
		\end{tabular}
		\caption{\agentone's policy $\pi^x_{\theta_1}$ for $\theta_1$}
  	\end{subfigure}
	\hfill
	\begin{subfigure}[b]{0.45\textwidth}
		\centering
		\begin{tabular}[h!]{c|c|c|}
			\diagbox{state}{action} & $a^x = a_1$ & $a^x = a_2$ \\ \hline
			$s = gold$   & $0$ & $1$ \\ \hline
			$s = end$  & $0.5$ & $0.5$  \\ \hline
		\end{tabular}
		\caption{\agentone's policy $\pi^x_{\theta_2}$ for $\theta_2$}		
  \end{subfigure}
  \caption{Policies for \agentone~for $\theta_1$ and $\theta_2$. For our proof, only the actions in state $s=gold$ are important. Hence, we have set $\pi^x_{\theta_1}(.~|~s=end)$ and $\pi^x_{\theta_2}(.~|~s=end)$ to be a uniform probability of picking actions.}
  \label{appendix.theorem1.fig-policy}
\end{figure*}

Next, in Figure~\ref{appendix.theorem1.fig-mdps}, we show two MDPs $\mdpM(\theta_1)$ and $\mdpM(\theta_2)$ as perceived by \agenttwo.  It is easy to see that the best response policies for \agenttwo~are given as follows: (i) for $\theta_1$, $\pi^*_{\theta_1}(a_1~|~s=gold) = 1$, and (ii) for $\theta_2$, $\pi^*_{\theta_2}(a_2~|~s=gold) = 1$. Any action can be taken from state $s=end$ as it brings zero reward and agent continues to stay in $s=end$. Also, these best response policies have a total reward given by $J_{\theta_1}(\pi^*_{\theta_1}) = \frac{\rmax}{1-\gamma}$ and $J_{\theta_2}(\pi^*_{\theta_2}) = \frac{\rmax}{1-\gamma}$.
\begin{figure*}[!h]
	\centering
	\begin{subfigure}[b]{0.45\textwidth}
		\centering
		\includegraphics[width=0.8\textwidth]{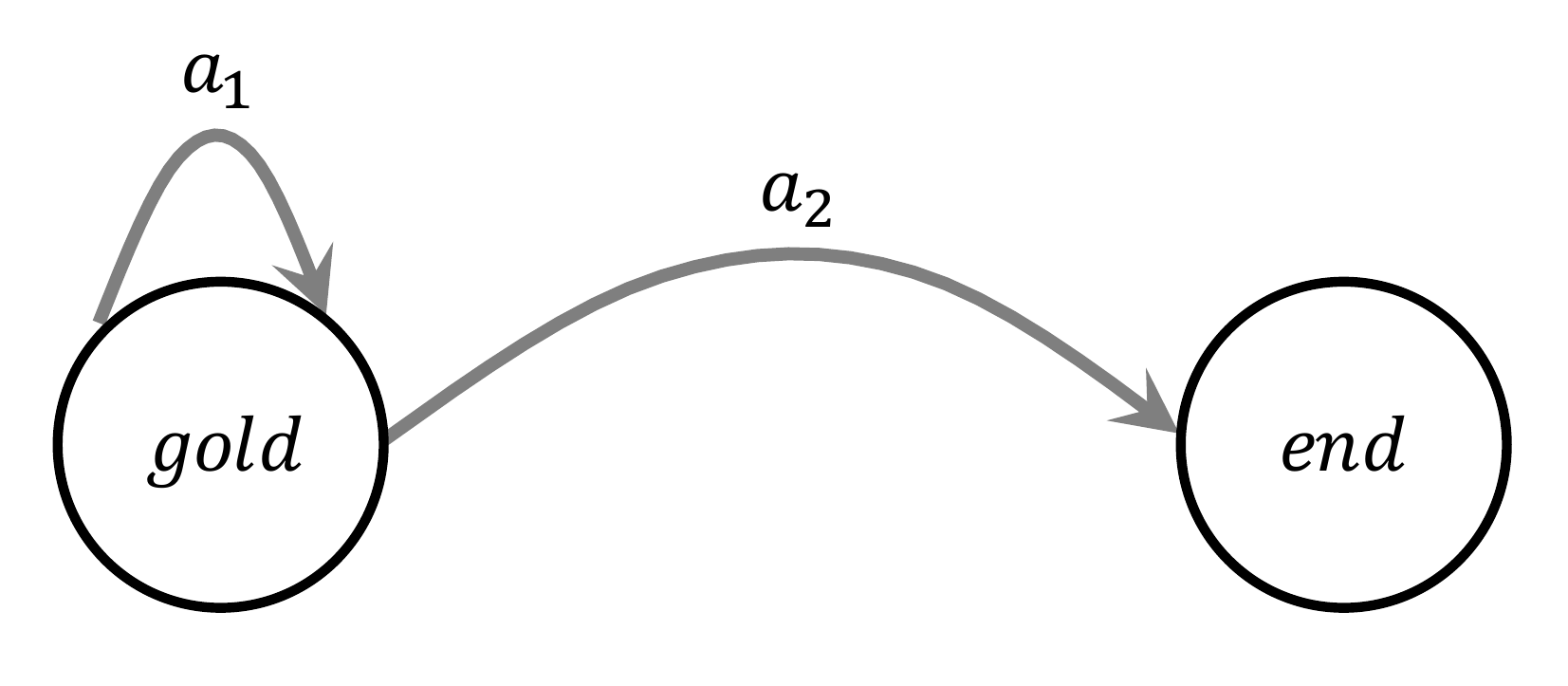}
		\caption{MDP $\mathcal{M}(\theta_1)$}
  	\end{subfigure}
	\hfill
	\begin{subfigure}[b]{0.45\textwidth}
		\centering
		\includegraphics[width=0.8\textwidth]{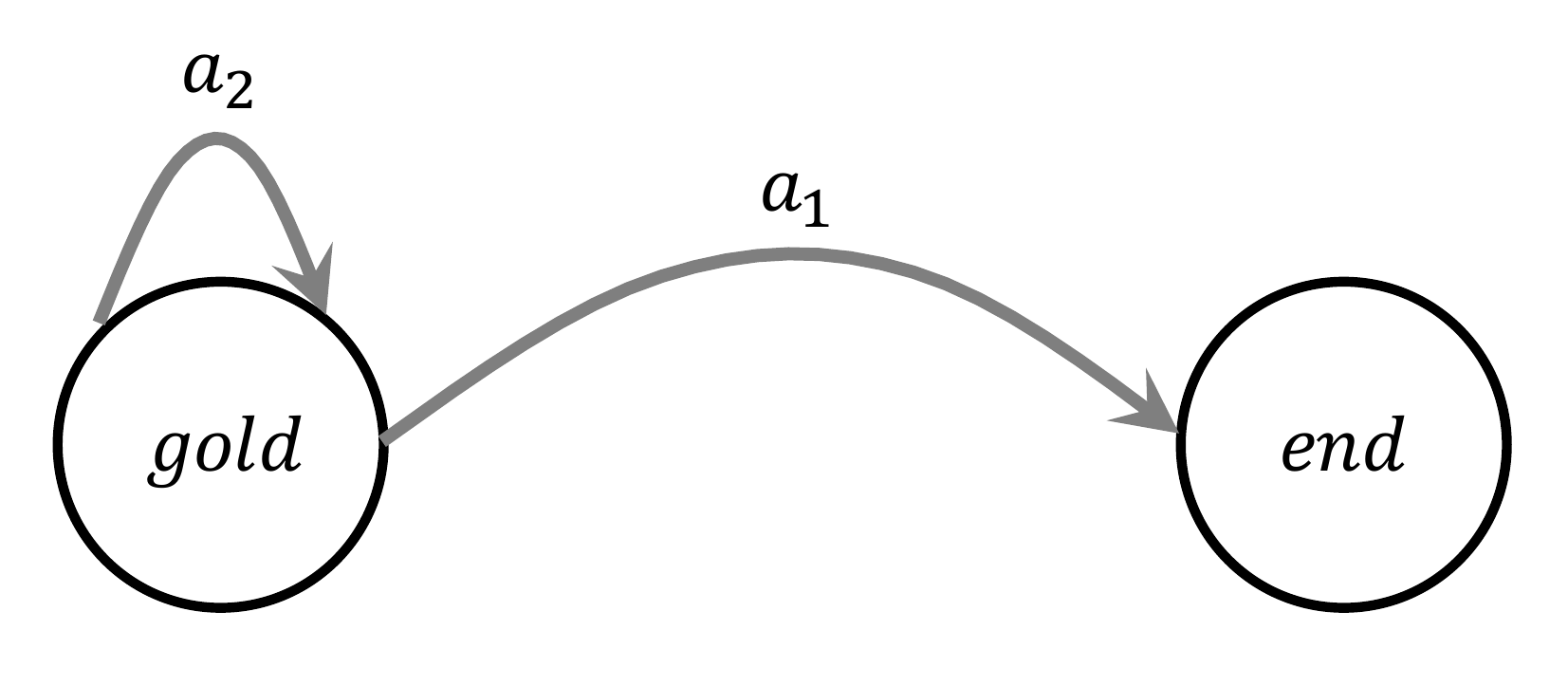}
		\caption{MDP $\mathcal{M}(\theta_2)$}
  	\end{subfigure}
  \caption{MDPs as perceived by \agenttwo~for $\theta_1$ and $\theta_2$. Note that arrows from the state $s=end$ are omitted as any action by \agenttwo~results in the next state as $s=end$.}		
  \label{appendix.theorem1.fig-mdps}
\end{figure*}

However, when the underlying parameter $\theta$ is unknown, the best response (in a maxmin sense) policy $\pi \in \Pi$ of \agenttwo{} as defined in Equation~\ref{maxmin_formulation_static} is given by:
\begin{align*}
\pi^*_{\Theta} = \argmin_{\pi \in \Pi} \max_{\theta \in \Theta} \Big(J_\theta(\pi^*_\theta) - J_\theta(\pi)\Big)
 \end{align*}

Next, we compute $\pi^*_{\Theta}$ for our problem instance. When considering the space of policies $\Pi$, it is enough to focus only on the state $s=gold$ and consider policies which take action $a_1$ from state  $s=gold$  with probability $p$ where $p \in [0, 1]$. For any such policy such that $\pi(a_1~|~gold) = p$, we can  compute  the following:
\begin{align*}
J_{\theta_1}(\pi) &= \frac{\rmax}{1-p \cdot \gamma} \quad \textnormal{and} \quad J_{\theta_2}(\pi) = \frac{\rmax}{1-(1 - p) \cdot \gamma}
\end{align*}

It can easily be shown that $\pi^*_{\Theta}$ is the policy given by $p=0.5$, i.e., $\pi^*_{\Theta}(a_1~|~ gold) = \pi^*_{\Theta}(a_2~|~gold) = 0.5$. Next we focus on the primary quantity of interest in the theorem, i.e., 
\begin{align*}
\max_{\theta \in \Theta} \Big(J_\theta(\pi^*_\theta) - J_\theta(\pi^*_\Theta)\Big)
\end{align*} 

As mentioned earlier, we have $J_\theta(\pi^*_\theta)  =  \frac{\rmax}{1-\gamma}$ for both $\theta_1$ and $\theta_2$. Also, it is easy to compute that $J_\theta(\pi^*_\Theta) = \frac{\rmax}{1 - \frac{\gamma}{2}} \leq 2\cdot\rmax$ for both $\theta_1$ and $\theta_2$. Hence, we can show that 
\begin{align*}
\max_{\theta \in \Theta} \Big(J_\theta(\pi^*_\theta) - J_\theta(\pi^*_\Theta)\Big) \geq  \frac{\rmax}{1-\gamma} - 2
\end{align*} 
which is arbitrary large when $\gamma$ is close to 1 or for large values of \rmax.



\end{proof}

%% file: 8.2_appendix-aaai-theorem2.tex
\section{Proof of Theorem~\ref{theorem:approximateMDP}}
In this section, we provide a proof of Theorem~\ref{theorem:approximateMDP}. The proof builds up on a few technical lemmas that we introduce first. 

\subsection{Approximately-equivalent MDPs}
First, we introduce a generic notion of approximately-equivalent MDPs and derive a few technical results for them that are useful to prove Theorem~\ref{theorem:approximateMDP}. This notion and technical results are adapted from the work by \cite{ApproximateEquival}.


\begin{definition} [approximately-equivalent MDPs, adapted from \cite{ApproximateEquival}] \label{appendix.def.equivalentMDPs}
	Suppose we have two MDPs $\mathcal{M}_1 = (S, A, T_1, R_1, \gamma, D_0)$ and $\mathcal{M}_2 = (S, A, T_2, R_2, \gamma, D_0)$,  and rewards are bounded in $[0,r_{max}]$. We call $\mathcal{M}_1$ and $\mathcal{M}_2$ as $(\epsilon_r, \epsilon_p)$ approximately-equivalent if the following holds: 
		\begin{align*}
			\max_{a \in A, s \in S} \norm{T_1(\cdot~|~s,a) - T_2(\cdot~|~s,a)}_1 \leq \epsilon_p \\
			\max_{a \in A, s \in S} |R_1(s, a) - R_2(s, a)| \leq \epsilon_r	\cdot \rmax 	
		\end{align*}
\end{definition}

Next, we state a useful technical lemma, which is adapted from the results of \cite{ApproximateEquival}.
\begin{lemma}\label{lm:eps-eq}
	Suppose we have two $(\epsilon_r, \epsilon_p)$ approximately-equivalent MDPs $\mathcal{M}_1$ and $\mathcal{M}_2$. Let $\pi_1$ and $\pi_2$ denote optimal policies (not necessarily unique) for $\mathcal{M}_1$ and $\mathcal{M}_2$ respectively. Let $V_{\mathcal{M}_1}^{\pi_1}$ denote the vector of value function per state for policy $\pi_1$ in MDP $\mathcal{M}_1$; similarly, $V_{\mathcal{M}_2}^{\pi_1}$ denotes the vector of value function per state for policy $\pi_1$ in MDP $\mathcal{M}_2$. We can bound these two vectors of value functions as follows:
	\begin{align}
		\norm{V_{\mathcal{M}_1}^{\pi_1} - V_{\mathcal{M}_2}^{\pi_1}}_\infty \leq \frac{\epsilon_r \cdot \rmax}{1-\gamma} + \frac{\gamma \cdot \epsilon_p \cdot \rmax}{(1-\gamma)^{2}}   
	\end{align}
\end{lemma}

\begin{proof}For ease of presentation of the key ideas, we will write this proof considering reward functions that only depend on the current state and not on actions taken; the proof can be easily extended to generic reward functions.

The proof idea is based on looking at intermediate outputs of a policy-iteration algorithm \cite{suttonbarto1998} when evaluating $\pi_1$  in $\mdpM_1$ and $\mdpM_2$. Let us consider an iteration $m$ of policy-iteration algorithm and let us use $V_{\mdpM_1, m}^{\pi_1}$ to denote the vector of value functions when evaluating $\pi_1$ in $\mdpM_1$. Similarly, we use $V_{\mdpM_2 ,m}^{\pi_1}$ to denote the vector of value functions when evaluating $\pi_1$ in $\mdpM_2$ at iteration $m$ of policy-iteration algorithm. Note that as $m \rightarrow \infty$, the vectors $V_{\mdpM_1, m}^{\pi_1}$ and $V_{\mdpM_2, m}^{\pi_1}$ converge to $V_{\mdpM_1}^{\pi_1}$ and $V_{\mdpM_2}^{\pi_1}$ respectively.


We prove the lemma using an inductive argument. Let $V_{\textnormal{max}}$ denote the maximum value of a state value function for any policy in $\mdpM_1$ or in $\mdpM_2$. We claim that if we start with same initial value functions $V_{\mdpM_1, 0}^{\pi_1}$ and $V_{\mdpM_2, 0}^{\pi_1}$ as initialization, in iteration $m$ for $m > 0$ we have the following:
\begin{align*}
	\norm{V_{\mdpM_1,m}^{\pi_1} - V_{\mdpM_2,m}^{\pi_1}}_\infty \leq \Big( \epsilon_r \cdot \rmax + \gamma \cdot \epsilon_p \cdot V_{\textnormal{max}} \Big) \cdot \Big(\sum_{i=0}^{m-1} \gamma^i\Big)    
\end{align*}
	
Without loss of generality, let us assume that we start with zero-valued vectors for $V_{\mdpM_1, 0}^{\pi_1}$ and $V_{\mdpM_2, 0}^{\pi_1}$ at initialization, we can then write:
\begin{align*}
		V_{\mdpM_1,1}^{\pi_1}(s) = R_1(s) + \sum_{a \in A} \pi_1(a~|~s) \cdot \gamma \cdot \sum_{s' \in S} T_1(s'~|~s,a) \cdot V_{\mdpM_1,0}^{\pi_1}(s') = R_1(s) \\ 
		V_{\mdpM_2,1}^{\pi_1}(s) = R_2(s) + \sum_{a \in A} \pi_1(a~|~s) \cdot \gamma \cdot \sum_{s' \in S} T_2(s'~|~s,a)\cdot V_{\mdpM_2,0}^{\pi_1}(s') = R_2(s)
\end{align*}
	
From above, we get the following base case of induction for $m=1$: 
\begin{align*}
	\norm{V_{\mdpM_1,1}^{\pi_1} - V_{\mdpM_2,1}^{\pi_1}}_\infty = \norm{R_1 - R_2}_\infty \leq \epsilon_r \cdot \rmax  \leq \epsilon_r \cdot \rmax + \gamma \cdot \epsilon_p  \cdot V_{\textnormal{max}}
\end{align*}
	
For completing the proof by induction, we will assume that the claim holds for $m$ and then we prove it for $m+1$. Using Bellman update we can write:
\begin{align}
	V_{\mdpM_1,m+1}^{\pi_1}(s) = R_1(s) + \sum_{a \in A} \pi_1(a~|~s) \cdot \gamma \cdot \sum_{s' \in S} T_1(s'~|~s,a) \cdot V_{\mdpM_1,m}^{\pi_1}(s') \label{appendix.th2.eq.Vt.1}\\ 
	V_{\mdpM_2,m+1}^{\pi_1}(s) = R_2(s) + \sum_{a \in A} \pi_1(a~|~s) \cdot \gamma \cdot \sum_{s' \in S} T_2(s'~|~s,a) \cdot V_{\mdpM_2,m}^{\pi_1}(s')  \label{appendix.th2.eq.Vt.2}
\end{align}
	
When we subtract two expressions \eqref{appendix.th2.eq.Vt.1} and \eqref{appendix.th2.eq.Vt.2}, we get:
\begin{align} 
&\Big|V_{\mdpM_1,m+1}^{\pi_1}(s) - V_{\mdpM_2,m+1}^{\pi_1}(s)\Big| \notag \\
&= \Big|R_1(s)-R_2(s)  + \sum_{a \in A} \pi_1(a~|~s) \cdot \gamma \cdot \sum_{s' \in S} \Big(T_1(s'~|~s,a) \cdot V_{\mdpM_1,m}^{\pi_1}(s') -T_2(s'~|~s,a) \cdot V_{\mdpM_2,m}^{\pi_1}(s')\Big)\Big| \notag \\
&\leq \Big|R_1(s)-R_2(s)\Big| + \Big|\sum_{a \in A} \pi_1(a~|~s) \cdot \gamma \cdot \sum_{s' \in S} \Big(T_1(s'~|~s,a) \cdot V_{\mdpM_1,m}^{\pi_1}(s') -T_2(s'~|~s,a) \cdot V_{\mdpM_2,m}^{\pi_1}(s')\Big)\Big| \notag \\
& =\Big|R_1(s)-R_2(s)\Big| + \Big|\sum_{a \in A} \pi_1(a~|~s) \cdot \gamma \cdot \sum_{s' \in S} \Big(T_1(s'~|~s,a) \cdot V_{\mdpM_1,m}^{\pi_1}(s') - T_2(s'~|~s,a) \cdot V_{\mdpM_1,m}^{\pi_1}(s') \notag \\
& \qquad \qquad \qquad \qquad \qquad \qquad \qquad \qquad \qquad  + T_2(s'~|~s,a) \cdot V_{\mdpM_1,m}^{\pi_1}(s') - T_2(s'~|~s,a) \cdot V_{\mdpM_2,m}^{\pi_1}(s')\Big)\Big| \notag \\
& = \Big|R_1(s)-R_2(s)\Big| + \Big|\sum_{a \in A} \pi_1(a~|~s) \cdot \gamma \cdot \sum_{s' \in S} \Big( V_{\mdpM_1,m}^{\pi_1}(s') \cdot \big(T_1(s'~|~s,a) - T_2(s'~|~s,a)\big) \notag \\
& \qquad \qquad \qquad \qquad \qquad \qquad \qquad \qquad \qquad  + T_2(s'~|~s,a) \cdot  \big(V_{\mdpM_1,m}^{\pi_1}(s') - V_{\mdpM_2,m}^{\pi_1}(s')\big)\Big)\Big| \notag \\
& \leq \Big|R_1(s)-R_2(s)\Big| + \sum_{a \in A} \pi_1(a~|~s) \cdot \gamma \cdot \sum_{s' \in S} \Big(V_{\mdpM_1,m}^{\pi_1}(s') \cdot \Big|T_1(s'~|~s,a)-T_2(s'~|~s,a)\Big|  \notag \\
& \qquad \qquad \qquad \qquad \qquad \qquad \qquad \qquad \qquad  + T_2(s'~|~s,a) \cdot \Big|V_{\mdpM_1,m}^{\pi_1}(s') - V_{\mdpM_2,m}^{\pi_1}(s')\Big|\Big) \notag \\
& \leq \epsilon_r \cdot \rmax + \sum_{a \in A} \pi_1(a~|~s) \cdot \gamma \cdot \sum_{s' \in S} \Big( V_{\textnormal{max}} \cdot \Big|T_1(s'~|~s,a)-T_2(s'~|~s,a)\Big| + T_2(s'~|~s,a) \cdot \norm{V_{\mdpM_1,m}^{\pi_1} - V_{\mdpM_2,m}^{\pi_1}}_\infty\Big) \notag \\
& \leq \epsilon_r \cdot \rmax + \sum_{a \in A} \pi_1(a~|~s) \cdot  \gamma \cdot \Big(V_{\textnormal{max}} \cdot  \max_{s'' \in S,a' \in A}\{ \norm{T_1(\cdot~|~s'',a')-T_2(\cdot~|~s'',a')}_1\}\Big) \notag \\
& \qquad \qquad \ \ \ + \sum_{a \in A} \pi_1(a~|~s) \cdot \gamma \cdot \Big(\sum_{s' \in S} T_2(s'~|~s,a) \cdot \norm{V_{\mdpM_1,m}^{\pi_1} - V_{\mdpM_2,m}^{\pi_1}}_\infty\Big) \notag \\
& = \epsilon_r \cdot \rmax + \gamma \cdot \epsilon_p \cdot V_{\textnormal{max}} + \sum_{a \in A} \pi_1(a~|~s) \cdot \gamma \cdot \Big( \sum_{s' \in S} T_2(s'~|~s,a) \cdot \norm{V_{\mdpM_1,m}^{\pi_1} - V_{\mdpM_2,m}^{\pi_1}}_\infty\Big)  \label{appendix.th2.eq.induction1}  \\
& \leq \epsilon_r \cdot \rmax + \gamma \cdot \epsilon_p \cdot V_{\textnormal{max}} + \sum_{a \in A} \pi_1(a~|~s) \cdot \gamma \cdot \bigg(\sum_{s' \in S}T_2(s'~|~s,a) \cdot \Big( \epsilon_r \cdot \rmax + \gamma \cdot \epsilon_p \cdot V_{\textnormal{max}}\Big) \cdot \Big(\sum_{i=0}^{m-1} \gamma^i\Big)\bigg)   \label{appendix.th2.eq.induction2} \\
& = \Big( \epsilon_r \cdot \rmax + \gamma\cdot \epsilon_p \cdot V_{\textnormal{max}} \Big) \cdot \Big(\sum_{i=0}^{m} \gamma^i\Big) \notag
\end{align}

Note that we went from expression \eqref{appendix.th2.eq.induction1} to \eqref{appendix.th2.eq.induction2} by using the induction assumption that the condition holds at iteration $m$. Now, with $m \rightarrow \infty$ at the convergence of policy iteration algorithm, we can write:
\begin{align*}
		\norm{V_{\mdpM_1}^{\pi_1} - V_{\mdpM_2}^{\pi_1}}_\infty & \leq \Big(\epsilon_r \cdot \rmax + \gamma \cdot \epsilon_p \cdot V_{\textnormal{max}}\Big)\cdot \Big(\sum_{i=0}^{\infty} \gamma^i\Big) \\
		& = \frac{\Big(\epsilon_r \cdot \rmax + \gamma \cdot \epsilon_p \cdot V_{\textnormal{max}}\Big)}{1-\gamma} \\
		 & \leq \frac{\epsilon_r \cdot \rmax}{1-\gamma} + \frac{\gamma \cdot \epsilon_p \cdot \rmax}{(1-\gamma)^2}    
\end{align*}
where we used the fact that $V_{\textnormal{max}} \leq \frac{\rmax}{1 - \gamma}$ when rewards are in the range $[0, \rmax]$. This completes the proof.
\end{proof}

\subsection{Relation between smoothness of \agentone~policies and perceived transition dynamics by \agenttwo}
For our problem setting with parameter space $\Theta$ and $\theta_1, \theta_2 \in \Theta$, let us define two MDPs $\mathcal{M}(\theta_1) := (S, A, T_{\theta_1}, R_{\theta_1}, \gamma, \mathcal{D}_0)$  and $\mathcal{M}(\theta_2) := (S, A, T_{\theta_2}, R_{\theta_2}, \gamma, \mathcal{D}_0)$. Let $\pi^x_{\theta_1}$ and $\pi^x_{\theta_2}$ denotes the policies for \agentone~in these two MDPs. In this section, we will provide technical lemmas which connect the distance between $\pi^x_{\theta_1}$ and $\pi^x_{\theta_2}$ to the approximate-equivalence of these two MDPs in terms of their transition dynamics $T_{\theta_1}$ and $T_{\theta_2}$.

Note that, for $\theta \in \Theta$, $T_{\theta}(s'~|~s, a)=\E_{a^x}[T^{x,y}(s'~|~s, a, a^x)]$, where $a^x \sim \pi^x_{\theta}(\cdot~|~s)$, i.e., $T_{\theta}(s'~|~s, a)$ corresponds to the transition dynamics derived from a two agent MDP for which \agentone{}'s policy is $\pi^x_{\theta}$. 

In Lemma~\ref{appendix.thm2.lm-smoothness-simple} below, we provide a simplified result that captures the relation between smoothness of \agentone~policies and perceived transition dynamics by \agenttwo. Then, we provide a more generic result in Lemma~\ref{appendix.thm2.lm-smoothness-generic} which also accounts for the influence property of two-agent MDP (see Equation~\eqref{eq.influence}).


\begin{lemma} \label{appendix.thm2.lm-smoothness-simple}
Consider two MDPs $\mathcal{M}(\theta_1) := (S, A, T_{\theta_1}, R_{\theta_1}, \gamma, \mathcal{D}_0)$  and $\mathcal{M}(\theta_2) := (S, A, T_{\theta_2}, R_{\theta_2}, \gamma, \mathcal{D}_0)$. Let $\pi^x_{\theta_1}$ and $\pi^x_{\theta_2}$ denotes the policies for \agentone in these two MDPs. Here, for $\theta \in \Theta$, $T_{\theta}(s'~|~s, a)=\E_{a^x}[T^{x,y}(s'~|~s, a, a^x)]$, where $a^x \sim \pi^x_{\theta}(\cdot~|~s)$, i.e., $T_{\theta}(s'~|~s, a)$ corresponds to the transition dynamics derived from a two agent MDP for which \agentone{}'s policy is $\pi^x_{\theta}$. Then, the following holds:
\begin{align}
		\max_{a \in A,s \in S} \norm{T_{\theta_1}(\cdot~|~s, a) - T_{\theta_2}(\cdot~|~s, a)}_1 \leq \max_{s \in S} \norm{\pi^x_{\theta_1}(\cdot~|~s) - \pi^x_{\theta_2}(\cdot~|~s)}_1 \label{appendix.thm2.lm-smoothness-simple.eq}
\end{align}
\end{lemma}
\begin{proof}
The proof follows by starting from the left-hand side of \eqref{appendix.thm2.lm-smoothness-simple.eq} and then applying a set of algebraic rules and properties of probability distributions to arrive at the right-hand side.

\begin{align*}
		&\max_{a \in A,s \in S} \norm{T_{\theta_1}(\cdot~|~s, a) - T_{\theta_2}(\cdot~|~s, a)}_1 \\
		& =  \max_{a \in A, s \in S}  \norm{\sum_{a^x \in A} \pi^x_{\theta_1}(a^x~|~s) \cdot T^{x,y}(\cdot~|~s,a,a^x) - \sum_{a^x \in A} \pi^x_{\theta_2}(a^x~|~s) \cdot T^{x,y}(\cdot~|~s,a,a^x)}_1 \\
		& = \max_{a \in A, s \in S}  \norm{\sum_{a^x \in A} T^{x,y}(\cdot~|~s,a,a^x)\cdot \Big(\pi^x_{\theta_1}(a^x~|~s)-\pi^x_{\theta_2}(a^x~|~s)\Big)}_1 \\
		& \leq  \max_{a \in A, s \in S}  \sum_{a^x \in A} \norm{T^{x,y}(\cdot~|~s,a,a^x)}_1 \cdot \Big|\pi^x_{\theta_1}(a^x~|~s) - \pi^x_{\theta_2}(a^x~|~s) \Big| \\
		& = \max_{a \in A, s \in S}  \sum_{a_x \in A} \Big|\pi^x_{\theta_1}(a^x~|~s) - \pi^x_{\theta_2}(a^x~|~s)\Big| \\
		& =  \max_{s \in S} \norm{\pi^x_{\theta_1}(\cdot~|~s) - \pi^x_{\theta_2}(\cdot~|~s)}_1
\end{align*}
	\end{proof}

Next, we will provide a more generic result which also accounts for the influence property of two-agent MDP (see Equation~\eqref{eq.influence}) as introduced by \cite{dimitrakakis2017multi} and also used subsequently in other works analysing two-agent MDPs (e.g., \cite{radanovic2019learning}). Recall that, this property captures how much one agent can affect probability distribution of next state with her actions as perceived by the second agent.  As defined in Equation~\eqref{eq.influence}, the influence property is given by: 
\begin{align*}
	\mathcal{I}_{x} := \max_{s \in S} \bigg( \max_{a}\max_{b,b'} \norm{T^{x,y}(.|s,a,b) - T^{x,y}(.|s,a,b')}_1 \bigg)
\end{align*}
where $a$ represents action of \agenttwo~, $b,b'$ represents two distinct actions of \agentone, and $T^{x,y}$ here represents the dynamics of a two-agent MDP as used above (also, see Section~\ref{sec.setup.model}). 

Below, we provide a more generic result of Lemma~\ref{appendix.thm2.lm-smoothness-simple} which additionally accounts for the influence. This lemma is based on the results by \cite{dimitrakakis2017multi} . 

\begin{lemma}[Adapted from \cite{dimitrakakis2017multi}]\label{appendix.thm2.lm-smoothness-generic}
Consider two MDPs $\mathcal{M}(\theta_1) := (S, A, T_{\theta_1}, R_{\theta_1}, \gamma, \mathcal{D}_0)$  and $\mathcal{M}(\theta_2) := (S, A, T_{\theta_2}, R_{\theta_2}, \gamma, \mathcal{D}_0)$. Let $\pi^x_{\theta_1}$ and $\pi^x_{\theta_2}$ denotes the policies for \agentone in these two MDPs. Here, for $\theta \in \Theta$, $T_{\theta}(s'~|~s, a)=\E_{a^x}[T^{x,y}(s'~|~s, a, a^x)]$, where $a^x \sim \pi^x_{\theta}(\cdot~|~s)$, i.e., $T_{\theta}(s'~|~s, a)$ corresponds to the transition dynamics derived from a two agent MDP for which \agentone{}'s policy is $\pi^x_{\theta}$. Then, the following holds:
\begin{align}
		\max_{a \in A,s \in S} \norm{T_{\theta_1}(\cdot~|~s, a) - T_{\theta_2}(\cdot~|~s, a)}_1 \leq \mathcal{I}_{x} \cdot \max_{s \in S} \norm{\pi^x_{\theta_1}(\cdot~|~s) - \pi^x_{\theta_2}(\cdot~|~s)}_1 \label{appendix.thm2.lm-smoothness-generic.eq}
\end{align}
\end{lemma}

Note that $\mathcal{I}_{x} \in [0, 1]$ and hence Lemma~\ref{appendix.thm2.lm-smoothness-generic} strictly generalizes  Lemma~\ref{appendix.thm2.lm-smoothness-simple}. Influence property allows us to account for more fine-grained aspects of the problem setting in the performance analysis. For instance, when $\mathcal{I}_{x} = 0$, then \agentone~does not affect the transition dynamics as perceived by \agenttwo~and we can expect to have better performance for \agenttwo. The proof of  Lemma~\ref{appendix.thm2.lm-smoothness-generic} follows along the basic ideas as used in the proof of Lemma~\ref{appendix.thm2.lm-smoothness-simple}, however, requires a more detailed application of rules to bring in the $\mathcal{I}_{x}$ in the bound \cite{dimitrakakis2017multi}.
\subsection{Putting it together to prove the theorem}
\begin{proof}
Let us begin by recalling the setting we are considering in the theorem. We have $\theta^{\textnormal{test}} \in \Theta$ as the type of \agentone~at test time and \agenttwo~uses a policy $\pi^*_{\hat{\theta}}$ such that $||\theta^{\textnormal{test}} - \hat{\theta}|| \leq \epsilon$. Parameters $(\alpha, \beta, \mathcal{I}_{x})$ characterize the smoothness as defined in Section~\ref{sec:performance-analysis}. $r_{\textnormal{max}}$ denotes the maximum value of reward. 

Below, we will show that the two MDPs $\mdpM(\hat{\theta})$ and $\mdpM(\theta^{\textnormal{test}})$ are $(\epsilon_r, \epsilon_p)$ approximately-equivalent (see Definition~\ref{appendix.def.equivalentMDPs}) for the following values of $\epsilon_r$ and $\epsilon_p$:
\begin{itemize}
	\item $\epsilon_r =  ||\theta^{\textnormal{test}} - \hat{\theta}|| \cdot \alpha \leq \epsilon \cdot \alpha$ given the smoothness assumption on the parametric MDP $\mathcal{M}(\theta)$  w.r.t. the rewards.
	\item $\epsilon_p \leq  \sqrt{2 \cdot \beta \cdot ||\theta^{\textnormal{test}} - \hat{\theta}||} \leq \sqrt{2 \cdot \beta \cdot \epsilon}$ given the smoothness assumption on policies for \agentone~w.r.t.\ parameter $\theta$. Here, we have used the Pinsker's inequality \cite{pinsker1964information} stating that If $P$  and  $Q$ are two probability distributions on a measurable space, then $\norm{P - Q}_1 ~\leq~ \sqrt{2 \cdot  \textrm{KL}(P,Q)}$.
\end{itemize}

Then, by applying the results from Lemma~\ref{lm:eps-eq} and Lemma~\ref{appendix.thm2.lm-smoothness-generic}, we get the desired result  that the total utility achieved by \agenttwo~ has the following guarantees:
\vspace{-2mm}
\begin{align*}
    J_{\theta^{\textnormal{test}}}(\pi^*_{\hat{\theta}}) \geq J_{\theta^{\textnormal{test}}}(\pi^*_{\theta^{\textnormal{test}}}) - \frac{\epsilon \cdot \alpha \cdot r_{\textnormal{max}}}{1-\gamma}
    - \frac{\mathcal{I}_{x} \cdot \sqrt{2 \cdot \beta \cdot \epsilon} \cdot r_{\textnormal{max}}}{(1-\gamma)^{2}}
\vspace{-2mm}
\end{align*}
\end{proof}